\documentclass[twoside]{article}

\usepackage{soul}
\usepackage{graphicx}
\usepackage{amsfonts}
\usepackage{mathtools}
\usepackage[colorlinks=true]{hyperref}
\usepackage{xcolor}
\definecolor{linkblue}{RGB}{0,0,128}
\hypersetup{
    linkcolor=linkblue,
    citecolor=linkblue,
    urlcolor=linkblue
}
\usepackage{bbm}
\usepackage{algorithm}
\usepackage{algpseudocode}
\usepackage{multirow}
\usepackage[normalem]{ulem}

\DeclareMathOperator*{\argmin}{argmin}
\usepackage[round]{natbib}

\usepackage{mathtools}
\mathtoolsset{showonlyrefs=true}

\newcommand{\GP}{\textsc{gp}}
\newcommand{\PROGP}{\textsc{pro-gp}}
\newcommand{\SVGP}{\textsc{svgp}}
\newcommand{\PPGPR}{\textsc{ppgpr}}
\newcommand{\IPROGP}{\textsc{i-pro-gp}}
\newcommand{\SIPROGP}{\textsc{i-pro-gp}\textsuperscript{\scriptsize{\textsc{s}}}}
\newcommand{\PIPROGP}{\textsc{i-pro-gp}\textsuperscript{\scriptsize{\textsc{p}}}}

\newcommand{\dd}{\mathrm{d}}
\newcommand{\mbf}{\mathbf}
\newcommand{\mbfu}{\mathbf{u}}
\DeclareMathOperator{\KL}{KL}

\usepackage{booktabs}
\usepackage{amsthm}
\newtheorem{assumption}{Assumption}
\newtheorem{theorem}{Theorem}

\newtheorem{lemma}{Lemma}

\newenvironment{talign*}
 {\csname align*\endcsname}
 {\endalign}

\usepackage[symbols,nogroupskip,sort=none]{glossaries-extra}
\newglossarystyle{mysymbols}{
  \setglossarystyle{long}
  
}
\setglossarystyle{mysymbols}
\glsxtrnewsymbol[description={Input space}]{S1}{$\mathcal{X}$}
\glsxtrnewsymbol[description={Response space}]{S2}{$\mathcal{Y}$}
\glsxtrnewsymbol[description={Observed data $\{(x_i, y_i) \}_{i=1}^n$}]{S4}{$\mathcal{D}_n$}
\glsxtrnewsymbol[description={Space of Borel probability measures}]{S3}{$\mathcal{P}$}
\glsxtrnewsymbol[description={Polish function space}]{F1}{$\mathcal{F}$}
\glsxtrnewsymbol[description={Prior measure on $\mathcal{F}$}]{F2}{$\Pi$}
\glsxtrnewsymbol[description={Conditional response model for $f \in \mathcal{F}$ and $x \in \mathcal{X}$}]{S5}{$P_f^x$}
\glsxtrnewsymbol[description={Density of conditional response model $P_f^x$ w.r.t Lebesgue measure}]{S6}{$p_f(\cdot \mid x)$}
\glsxtrnewsymbol[description={Collection of conditional responses $\{P_f^x : x \in \mathcal{X}\}$}]{S7}{$P_f$}
\glsxtrnewsymbol[description={True conditional responses}]{S8}{$P_0$}
\glsxtrnewsymbol[description={Family of nonparametric models $\{P_f : f \in \mathcal{F} \}$}]{S9}{$\mathcal{M}_{\mathcal{F}}$}
\glsxtrnewsymbol[description={Bayes posterior given data $\mathcal{D}_n$}]{S12}{$Q_n^B$}
\glsxtrnewsymbol[description={Predictively Oriented posterior given data $\mathcal{D}_n$}]{S13}{$Q_n^{\mathrm{PrO}}$}
\glsxtrnewsymbol[description={Posterior predictive measure $\int_{\mathcal{F}} P_f^x Q(\dd f)$ at $x$ given $Q \in \mathcal{P}
(\mathcal{F)}$}]{S14}{$P_Q^x$}
\glsxtrnewsymbol[description={Kullback-Leibler divergence}]{S15}{\textnormal{KL}}
\glsxtrnewsymbol[description={Gaussian density with variance v}]{S17}{$\varphi_v$}
\glsxtrnewsymbol[description={Probability measure on $\mathbb{R}^d$}]{S18}{$\tau$}
\glsxtrnewsymbol[description={Prior-conditional lift $\int_{\mathbb{R}^d} \Pi(\dd f \mid \mbf{u}) \tau(\dd \mbfu)$}]{S19}{$Q^\tau$}
\glsxtrnewsymbol[description={PrO-GP objective (exact or inducing)}]{S20}{$\mathcal{J}$}
\glsxtrnewsymbol[description={Proper scoring rule $S: \mathcal{P}(\mathcal{Y}) \times \mathcal{Y} \to \mathbb{R} \cup \{\infty \}$} ]{S21}{$S$}
\glsxtrnewsymbol[description={Function-space PrO objective with prior $\Pi$ and scoring rule $S$}]{S22}{$\mathcal{J}^{\Pi}_S$}

\usepackage[preprint]{aistats2027}

\begin{document}

%

%

\twocolumn[

\aistatstitle{Predictively Oriented Gaussian Process Posteriors}

\aistatsauthor{ Callum Lau \And Jeremias Knoblauch \And  Louis Sharrock }

\aistatsaddress{ University College London \And  University College London \And University College London } ]

\begin{abstract}
    Gaussian Processes (GPs) are a powerful tool for modelling and quantifying uncertainty in functional relationships. However, they require practitioners to make a number of design decisions, such as the choice of the kernel and the observation model. Suboptimal choices can produce misspecified models that do not capture the underlying data generating process. We introduce Predictively Oriented Gaussian Processes (PrO-GPs), which treat predictive uncertainty as the primary inferential target and provide a robust alternative to standard GPs. Although direct computation of a PrO posterior for nonparametric models is intractable, we derive a reduced formulation and practical sampling scheme for efficient computation. Through synthetic and real data experiments, we show that PrO-GPs produce better calibrated predictive distributions under model misspecification compared to standard GP approaches.
\end{abstract}

\section{INTRODUCTION}

Gaussian processes (GPs) provide a flexible framework for modelling functional relationships \citep{rasmussen2006gaussian}. Through their covariance kernel, GPs can encode prior beliefs about the structure of those relationships, making them well-suited to settings in which data is scarce or expensive to obtain. Hence, GPs are widely used in applications such as reconstructing geospatial fields from sparse measurements \citep{hugonnet2021accelerated, berlinghieri2023gaussian}, regressing atomistic properties of materials \citep{deringer2021gaussian}, and forecasting astronomical time series \citep{aigrain2023gaussian}.

\begin{figure}[h!]
\centering
\includegraphics[width=0.45\textwidth]{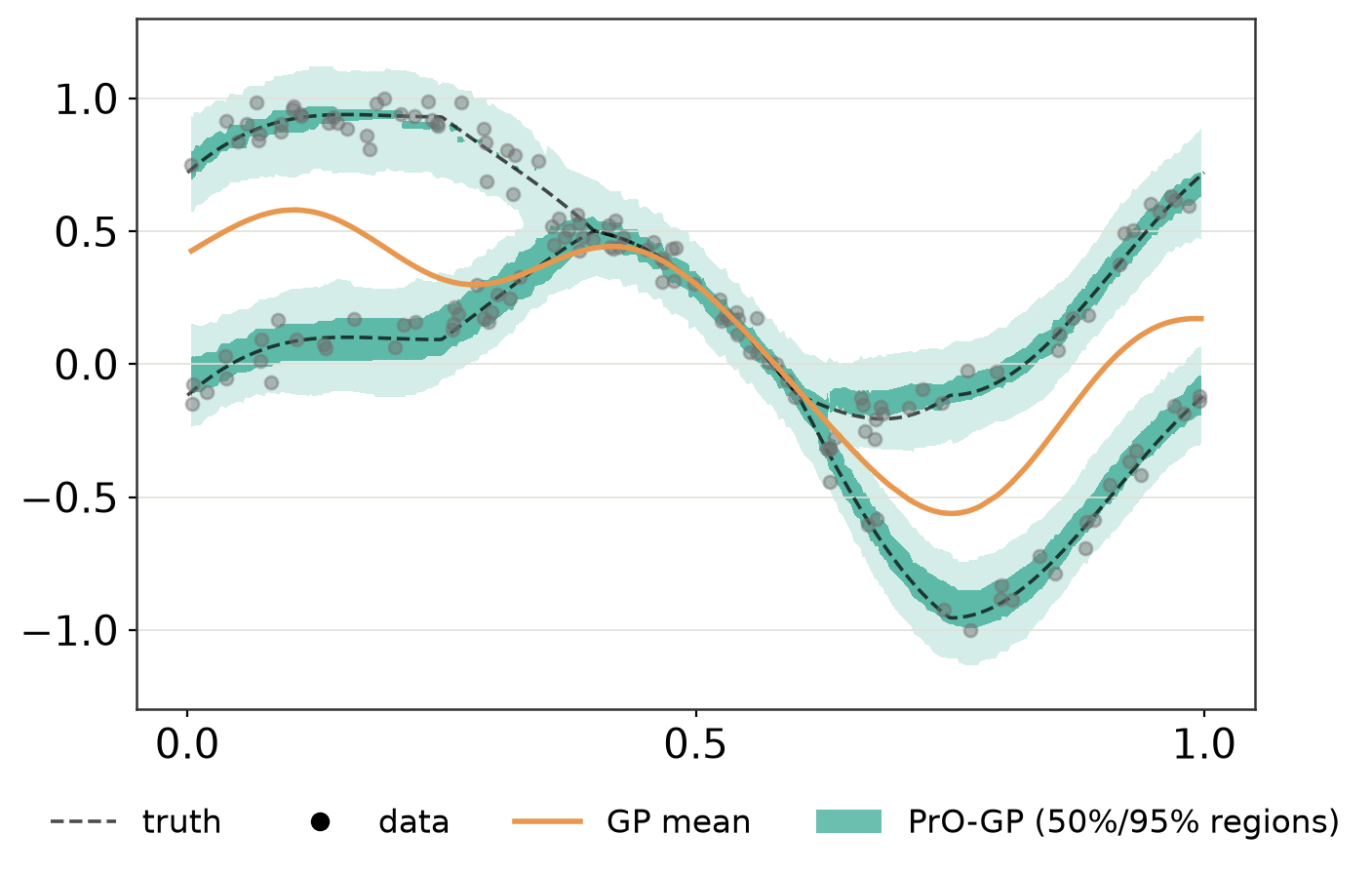}
    \vspace{-1mm}
    \caption{\textit{Bayesian GP and PrO-GP posterior predictive on multimodal data}. The GP posterior is centred around a single compromise function, with its predictive mean lying between the two modes. The PrO-GP instead recovers a mixture of functions, assigning predictive mass to both branches.}
    \label{fig:multimodal}
\end{figure}

 However, correctly specifying the covariance kernel and likelihood model is challenging and greatly influences the resulting predictive distributions \citep{wilson2013gaussian}. Indeed, downstream decisions based on GPs are generally non-robust to kernel choice \citep{stephenson2022measuring}, and likelihood misspecification can lead to poor approximations of the data generating process (DGP) \citep{jylanki2011robust}. 

We introduce Predictively Oriented Gaussian Processes (PrO-GPs), a new class of GPs, to remain robust to model misspecification. Building on the framework of PrO posteriors \citep{ mclatchie2025predictively, shen2025prediction}, PrO-GPs seek a distribution over latent functions for which the induced posterior predictive best describes the data. This contrasts the Bayesian posterior which, under suitable regularity conditions, concentrates around the \textit{single} function in the assumed model class that best approximates the DGP \citep{kleijn2006misspecification}. Consequently, the associated predictive distribution asymptotically behaves like that induced by this single best-fitting function, even when no individual function in the model class provides an adequate predictive description of the data. See Figure~\ref{fig:multimodal} for an illustrative example.

\paragraph{Related Work} The predictive approach to GP regression has been explored previously. \citet{jankowiak2020parametric} introduce PPGPR, which directly optimises the posterior predictive distribution of variational GPs, while \cite{jankowiak2020deep} extend this predictive approach to Deep GPs. Similarly, \cite{wei2021direct} study direct predictive optimisation of GPs under both log and squared losses.  
More broadly, \citet{kock2026predictive} apply the framework of Predictive Variational Inference \citep{lai2024predictive} to GP regression. However, these approaches impose a Gaussian or finite Gaussian-mixture parametric restriction on the induced predictive distribution; PrO-GPs target an unrestricted objective, allowing the model to retain non-Gaussian uncertainty over the latent function.

Non-robustness in GPs to misspecification of the prior or likelihood is well known \citep{wynne2021convergence}. Prior misspecification can arise when the kernel does not capture the structural properties of the unknown function, motivating automated kernel discovery methods \citep{duvenaud2013structure} and more expressive kernel families \citep{wilson2016deep}. Likelihood misspecification commonly occurs when the convenient Gaussian noise assumption is violated due to heavy tails \citep{jylanki2011robust}, heterogeneity \citep{le2005heteroscedastic} or mixture likelihoods \citep{lu2023robust}. A related line of work leverages generalised Bayesian inference \citep{bissiri2016general, knoblauch2022optimization, wild2022generalized} and score-matching losses to construct GP posteriors which are robust to outliers whilst preserving conjugacy \citep{pmlr-v235-altamirano24a, laplante2025robust, rooijakkers2025multi}. Our approach is distinct: rather than changing the kernel, likelihood or loss to target a specific form of misspecification, we modify the inferential procedure itself to provide additional flexibility to recover the DGP.
\paragraph{Contributions} 
We introduce PrO-GPs, a new class of GPs which remain robust to model misspecification. Our contributions are threefold:

\vspace{-1mm}
\begin{itemize}
    \item We establish an exact feature-reduction theorem which transforms optimisation over probability measures on function space into optimisation over probability measures on a finite-dimensional feature space, providing a tractable Euclidean representation of PrO-GP posteriors.
    \item We introduce a principled inducing-feature approximation and provide an efficient sampling scheme for scalable optimisation of the Euclidean objective.
    \item We show PrO-GPs can significantly improve predictive uncertainty quantification on synthetic misspecification and real-data regression tasks.
\end{itemize}

\section{BACKGROUND} \label{sec:Background}
Let $\mathcal{X}$ denote the input space and let $\mathcal{Y}=\mathbb{R}$ 
denote the response space. We observe a dataset $\mathcal{D}_n\coloneq\{(x_i,y_i)\}_{i=1}^n$,
and consider the scalar regression 
\begin{equation}
    y_i=f_0(x_i)+\varepsilon_i,
    \label{eq:regression-dgp}
\end{equation}
where $f_0:\mathcal{X}\to\mathbb{R}$ is an unknown true latent function and
$\varepsilon_1,\ldots,\varepsilon_n$ are independent observation errors. Let $P_0^x\in\mathcal{P}(\mathcal{Y})$ denote the true conditional distribution of the response given $x$, and write $P_0\coloneq\{P_0^x:x\in\mathcal{X}\}$.
Given a function $f\in\mathcal{F}$ and an assumed error model, let
$P_f^{x} \coloneq P_f(\cdot\mid x) \in \mathcal{P}(\mathcal{Y})$ denote the corresponding model conditional distribution of the response and $p_f(\cdot \mid x)$ its density with respect to the Lebesgue measure. Also let $P_f \coloneq \{P_f^{x} : x \in \mathcal{X} \}$. This defines the model class
\[
    \mathcal{M}_{\mathcal{F}}
    \coloneq
    \{P_f:f\in\mathcal{F}\}.
\]
We say that the model is well-specified if $P_0\in\mathcal{M}_{\mathcal{F}}$, and misspecified otherwise. We refer the reader to Appendix \ref{ap:notation} for the notation used throughout the paper.

\subsection{Gaussian Processes} \label{sec:gps}

A GP defines a probability measure over functions through
a mean function $m:\mathcal{X}\to\mathbb{R}$ and a positive-definite covariance
kernel $k:\mathcal{X}\times\mathcal{X}\to\mathbb{R}$ \citep{rasmussen2006gaussian}. We write $f\sim\mathcal{GP}(m,k)$
and denote the corresponding probability measure on a suitable measurable function space $\mathcal{F}$ by $\Pi$. 

For a finite collection of inputs
$X=(x_1,\ldots,x_n)$, let $\mathbf{f}_X\coloneq\bigl(f(x_1),\ldots,f(x_n)\bigr)^\top.$
The defining property of a GP is that $\mathbf{f}_X\sim\mathcal{N}(\mathbf{m}_{X},K_{XX})$,
where $(\mathbf{m}_X)_i=m(x_i)$ and $(K_{XX})_{i,j}=k(x_i,x_j).$ For any additional input $x$, write $ \mathbf k_{x X} = \bigl(k(x_1,x),\ldots,k(x_n,x)\bigr)^\top.$ Then conditionally on $\mathbf{f}_X$, $f$ is also a GP: $f \mid \mathbf{f}_X \sim \mathcal{GP}(\mu_{\mathbf{f}_X}, s_X)$ with mean function and covariance kernel
\begin{align}
    \mu_{\mathbf{f}_X}(x) &= m(x) + \mathbf{k}_{x X}^\top K_{XX}^{-1}(\mathbf{f}_X - \mathbf{m}_X) \notag \\
s_X(x,x')
    &=
    k(x, x')
    -
    \mathbf k_{x X}^\top
    K_{XX}^{-1}
    \mathbf k_{X x'}.
    \label{eq:sn_squared}
\end{align}

Hence for a single additional input $x_\star$ with realisation $f_{\star}$, this gives $f_\star \mid \mathbf{f}_X \sim \mathcal{N}(\mu_{\mathbf{f}_X}(x_\star), s_X(x_\star, x_\star))$. 

In the regression setting (\ref{eq:regression-dgp}), we observe noisy realisations $\mathbf{y} = (y_1, \ldots, y_n)^\top$ and are interested in computing the posterior predictive distribution $y_\star \mid x_\star, X, \mathbf{y}$. Bayesian inference does so via a two-step procedure. First, construct the posterior over $\mathcal{F}$, which we denote $Q_n^B$. This is obtained by updating the GP prior $\Pi$ using the likelihood, which corresponds to
\begin{equation}
    Q_n^{\mathrm{B}}(\mathrm{d}f)
    \coloneq
    \Pi(\mathrm{d}f\mid X, \mathbf{y})
    \propto
    \left\{
        \prod_{i=1}^n
        p_f(y_i\mid x_i)
    \right\}
    \Pi(\mathrm{d}f).
    \label{eq:bayesian-function-posterior}
\end{equation}
Second, average the conditional response distribution at $x_\star$ over the posterior $Q_n^B$:
\begin{equation*}
    P_{Q_n^{\mathrm{B}}}^{x_\star}
    =
    \int_{\mathcal{F}}
    P_f^{x_\star}\,
    Q_n^{\mathrm{B}}(\mathrm{d}f).
    \label{eq:bayesian-posterior-predictive}
\end{equation*}

In the special case of a homoscedastic Gaussian response model $P_f^{x}
    =
    \mathcal N\bigl(f(x),\sigma^2\bigr)
$, the GP prior is conjugate and hence the posterior predictive is Gaussian.


\subsection{Variational form of Bayesian posteriors}

We have shown that Bayesian inference first constructs a posterior distribution over functions $Q^B_n$, from which the posterior predictive follows as a mixture. More generally, any probability measure $Q\in\mathcal{P}(\mathcal{F})$
over functions induces the predictive
\begin{equation}
\begin{aligned}
    P_Q^{x}
    \coloneq
    \int_{\mathcal{F}}
    P_f^{x}\,Q(\mathrm{d}f).
    \label{eq:function-mixture-predictive}
\end{aligned}
\end{equation}
In this work, we are interested in obtaining a posterior $Q$ which is \textit{predictively oriented}, meaning the effect of $Q$ on $P_Q$ is the primary object of concern. 

To do so, we introduce the variational characterisation of the Bayesian posterior
\citep{donsker1975asymptotic,zellner1988optimal}. In particular, when
$\lambda_n=n$,  $Q_n^{\mathrm B}$ is the solution of
\begin{equation}
\begin{aligned}
\argmin_{Q\in\mathcal P(\mathcal F)}
\bigg\{
-\frac{\lambda_n}{n}
\sum_{i=1}^n
\int_{\mathcal F}
\log p_f(y_i\mid x_i)\,
Q(\dd f) 
+
\mathrm{KL}(Q\Vert\Pi)
\bigg\}.
\label{eq:variational-bayes-posterior}
\end{aligned}
\end{equation}
See Appendix~\ref{ap:bayes-as-gibbs} for a derivation. Each function $f$ is assessed through the score of its conditional predictive distribution $P_f^{x_i}$, and one targets a distribution $Q$ over $f$ which minimises these function-wise losses, subject to a KL-penalty which regularises deviations from the prior $\Pi$. Under suitable regularity conditions, the Bayesian posterior concentrates on a single loss-minimising function when $n \to \infty$ \citep{kleijn2006misspecification}, even when the model is misspecified.

Lastly, we note that \eqref{eq:variational-bayes-posterior} sits within the wider framework of Gibbs posteriors \citep{bissiri2016general}, where replacing the logarithmic score with other proper scoring rules \citep{gneiting2007strictly} yields general Bayesian updates which target alternative divergences to the KL divergence \citep{jewson2018principles}. For example, \cite{pmlr-v235-altamirano24a} introduce score-matching losses for robustness in GPs under Huber contamination \citep{huber1992robust}, whilst \citet{shen2025prediction} use the maximum mean discrepancy \citep{JMLR:v13:gretton12a} to derive a predictively oriented posterior for parametric models. Although our most general objective in Appendix \ref{ap:theoretical-proofs}, and the alternative sampling scheme in Appendix \ref{ap:wgf-sampling} make extensions to other scoring rules possible, in this work we focus on the logarithmic score due to its favourable computational properties.

\section{METHODOLOGY} \label{sec:Method}

We now develop a predictively oriented posterior directly on function space.
Rather than constructing a posterior over functions by scoring each function individually, we seek a probability measure over functions whose induced predictive distribution best describes the observed data.
The resulting optimisation problem is defined in terms of probability measures over (infinite-dimensional) functions, and is therefore not immediately amenable to computation. First, we introduce the function-space PrO-GP objective in Section \ref{sec:pro-gps}. Subsequently, we derive a formulation which reduces this objective down to optimisation over probability measures on a finite-dimensional feature space in Section \ref{sec:exact-feature-reduction}. For scalable computation on large datasets, we introduce an inducing-point approximation in Section \ref{sec:inducing-feature-approximation}. Lastly, we present an efficient sampling algorithm in Section \ref{sec:computation} and discuss guidance for hyperparameter optimisation in Section \ref{sec:hyperparam-optimisation}.

\subsection{Predictively Oriented GP posteriors}
\label{sec:pro-gps}

Let $f \sim \mathcal{GP}(m, k)$ with prior measure $\Pi$ on a suitable measurable function space $\mathcal{F}$. Suppose $Y_i \mid f, x_i \sim P_f^{x_i} = \mathcal{N}\left(f(x_i), \sigma^2\right)$ for $i=1,\ldots,n$. From \eqref{eq:function-mixture-predictive}, recall that any $Q\in\mathcal P(\mathcal F)$ induces the mixture predictive
\(
P_Q^x
\coloneq
\int_{\mathcal F}P_f^x\,Q(\mathrm df).
\)
Motivated by \citet{mclatchie2025predictively}, we define a function-space PrO-GP posterior $Q_n^{\text{PrO}}$ as any solution of
\begin{equation}
\argmin_{Q\in\mathcal P(\mathcal F)}
\left\{
-\frac{\lambda_n}{n}
\sum_{i=1}^n
\log p_Q(y_i \mid x_i)
+
\KL(Q\Vert\Pi)
\right\}.
\label{eq:function-pro}
\end{equation}

The crucial distinction from the Bayesian posterior is the order of scoring and averaging.
In particular, while the Bayesian objective averages the individual scores,  the PrO objective scores the predictive law directly. 
Consequently, PrO inference optimises over the broader class of predictive laws induced by probability measures on $\mathcal F$, rather than only the individual predictive laws $P_f$. 

For example, consider the DGP $y_i = f_{2Z_i-1}(x_i) + \varepsilon_i$ where $Z_i \sim \operatorname{Bernoulli}(\pi)$ is some unobserved binary variable. 
Then the true conditional distribution is $P_0(\cdot \mid x) = (1-\pi) \mathcal{N}(f_{-1}(x), \sigma^2) + \pi\mathcal{N}(f_1(x), \sigma^2)$.
However, our working model (\ref{eq:regression-dgp}) assumes all observations are generated by one common latent function.
Hence, the Bayes posterior concentrates on a single compromise function, whereas the PrO posterior can capture uncertainty about the latent functional regime that generated each observation.

\citet{mclatchie2025predictively} studied the statistical properties of \eqref{eq:function-pro} for general scoring rules and in finite-dimensional parameter spaces, $\mathcal{F} = \mathbb{R}^d$. In particular, under suitable regularity conditions, PrO posteriors \textit{predictively dominate} Bayesian (and Gibbs) posteriors: if the model is well-specified, the asymptotic predictive performance of the PrO posterior is equivalent to the Bayesian posterior, but for a \textit{non-trivially misspecified}  model \citep[Definition 2]{mclatchie2025predictively} the predictive performance is strictly superior. The aforementioned example is precisely such a case: define $\nu_X \in \mathcal{P}(\mathcal{X})$ for the marginal distribution of the covariates and the expected score
\begin{equation*}
\mathcal{S}(P_Q, P_0)
\coloneq
\mathbb{E}_{\substack{
Y \sim P_0^X \\
X \sim \nu_X
}}
\left[-\log p_Q(Y \mid X)\right].
\end{equation*}
Then we have  $P_0 \notin \mathcal{M}_\mathcal{F}$ and for some $Q \in \mathcal{P}(\mathcal{F})$
\begin{equation*}
    \inf_{Q \in \mathcal{P}(\mathcal{F})} \mathcal{S}(P_Q, P_0) < \inf_{f \in \mathcal{F}} \mathcal{S}(P_f, P_0),
\end{equation*}
Our concern here is how to define and compute the corresponding object when the parameter is an entire function.

\subsection{Exact feature reduction} \label{sec:exact-feature-reduction}

The optimisation in \eqref{eq:function-pro} is over probability measures on an infinite-dimensional function space. At first sight, this appears substantially more difficult than its finite-dimensional counterpart.
However, it can be significantly simplified since the predictive law only depends on $f$ through the finite feature map $ \mathbf{f}_X$. We present the following reduction theorem.

\begin{theorem}[Exact GP reduction]
\label{thm:exact-gp-reduction} 
Define
$
\mbf u \coloneq \mbf f_X
$
such that 
$
\mbf{u} \sim\pi_X \coloneq \mathcal N(\mbf m_X,K_{XX}).
$
For any $\tau\in\mathcal P(\mathbb R^n)$ with $\KL(\tau\Vert\pi_X)<\infty$, and $i=1,\ldots,n$ let
\begin{equation*}
    P_\tau^{x_i}(\dd y)
    \coloneq
    \int_{\mathbb R^n}
    \mathcal N(\dd y;\mbf{u}_i,\sigma^2)
    \tau(\dd\mbf u).
    \label{eq:exact-reduced-predictive}
\end{equation*}
Then minimising the function-space PrO-GP objective in \eqref{eq:function-pro} is equivalent to
\begin{equation}
    \argmin_{\tau\in\mathcal P(\mathbb{R}^n)}
    \left\{
    -\frac{\lambda_n}{n}
    \sum_{i=1}^n
    \log p_\tau(y_i \mid x_i)
    +
    \KL(\tau \Vert \pi_X)
    \right\}.
    \label{eq:exact-pro-gp-objective}
\end{equation}
Moreover, if $\tau_n$ minimises the finite-dimensional objective, then a corresponding function-space minimiser is
\begin{equation}
    Q_n^{\mathrm{PrO}}(\dd f)
    =
    \int_{\mathbb R^n}
    \Pi(\dd f\mid  \mbf u)
    \tau_n( \dd\mbf u).
    \label{eq:fs-pro-posterior}
\end{equation}
\end{theorem}
Importantly, Theorem \ref{thm:exact-gp-reduction} replaces optimisation over probability measures on the function space $\mathcal F$ by optimisation over probability measures on the Euclidean feature space $\mathbb R^n$.
In fact, this is an application of a far more general result which establishes a Euclidean reduction for \textit{any} nonparametric model whose predictive law only depends on $f$ through a finite feature map and any proper scoring rule, and our result is obtained as a simple corollary (see Appendix \ref{ap:theoretical-proofs}).

\paragraph{Prediction} Recall the PrO posterior is obtained via (\ref{eq:fs-pro-posterior}) and the corresponding posterior predictive is given in (\ref{eq:function-mixture-predictive}). Hence, for a test input $x_\star$, posterior prediction reduces to integrating standard GP conditional formulae (\ref{eq:sn_squared}) against the optimized law $\tau_n$. This yields
\begin{align}
    P_{Q_n^{\text{PrO}}}^{x_\star}
     &= \int_{\mathbb{R}^n} \mathcal{N}(\mu_{\mbf{u}}(x_\star), s_X(x_\star, x_\star) + \sigma^2)  \tau_n(\mathrm{d} \mbf{u}).
     \label{eq:pro-predictive}
\end{align}
For a detailed derivation see Appendix \ref{ap:gp-posterior-prediction}. Notably, even with a GP prior and Gaussian observation model, the PrO-GP predictive will generally be a non-Gaussian mixture because $\tau_n$ is generally non-Gaussian.

Whilst explicit mixtures of GPs (MoGPs) \citep{tresp2000mixtures} provide an alternative route to flexible predictive distributions such as \eqref{eq:pro-predictive}, we emphasise that \PROGP{}s are conceptually distinct. 
MoGPs enrich the statistical model itself by introducing multiple latent GP components. 
In contrast, PrO-GPs retain a single GP prior and likelihood and instead modify the inferential procedure, thus requiring neither gating mechanisms \citep{rasmussen2001infinite}, an explicit number of mixture components \citep{lazaro2012overlapping}, nor separately specified GP components \citep{nguyen2014fast}. 
Consequently, non-Gaussian predictive uncertainty emerges only when supported by the data, rather than being specified a priori through an explicit mixture structure.

\subsection{Inducing-point approximation} \label{sec:inducing-feature-approximation} 
Although the exact feature reduction makes it possible to apply existing computational methods, the dimension of the optimisation problem still grows with the number of observations. 
To obtain a scalable construction, we therefore restrict the admissible function-space laws through a lower-dimensional collection of inducing features and provide the following theorem.

\begin{theorem}[Inducing-point GP restriction]
\label{thm:inducing-point-gp-restriction}
Let $\Lambda=(\lambda_1,\ldots,\lambda_m)$ be a collection of $m$ inducing points and let
$\mbf u \coloneq \mbf f_\Lambda$, with induced prior
$
    \pi_\Lambda
    =
    \mathcal N(\mbf m_\Lambda,K_{\Lambda\Lambda}).
$
For any $\tau\in\mathcal P(\mathbb R^m)$ with
$\KL(\tau\Vert\pi_\Lambda)<\infty$, define
\[
    Q^\tau(\dd f)
    \coloneq
    \int_{\mathbb R^m}
    \Pi(\dd f\mid  \mbf{u})
    \tau(\dd\mbf u),
\]
and the restricted family
\[
    \mathcal Q_\Lambda
    \coloneq
    \left\{
        Q^\tau :
        \tau\in\mathcal P(\mathbb R^m),
        \ \KL(\tau\Vert\pi_\Lambda)<\infty
    \right\}.
\]
Define $P_\tau^{x_i}$ as the conditional predictive distribution induced by $Q^\tau$:
\begin{equation}
    P_\tau^{x_i}
    \coloneq
    P_{Q^\tau}^{x_i}
    =
    \int_{\mathbb R^m}
    \left[
        \int_{\mathcal F}
        P_f^{x_i}\,
        \Pi(\dd f\mid \mbf u)
    \right]
    \tau(\dd\mbf u).
    \label{eq:inducing-mixture-predictive}
\end{equation}
Then minimising the function-space PrO-GP objective in \eqref{eq:function-pro} over the restricted family $\mathcal{Q}_{\Lambda}$ is equivalent to
\begin{equation}
    \argmin_{\tau\in\mathcal P(\mathbb{R}^m)}
    \left\{
    -\frac{\lambda_n}{n}
    \sum_{i=1}^n
    \log p_\tau(y_i \mid x_i)
    +
    \KL(\tau \Vert \pi_\Lambda)
    \right\}.
    \label{eq:inducing-pro-gp-objective}
\end{equation}
Moreover, if $\tau_{\Lambda,n}$ minimises the finite-dimensional objective, then a corresponding restricted function-space minimiser is
\[
    Q_{\Lambda,n}^{\mathrm{PrO}}(\dd f)
    =
    \int_{\mathbb R^m}
    \Pi(\dd f\mid  \mbf u)
    \tau_{\Lambda,n}(\dd\mbf u)
\]
\end{theorem}
This approximation is in the spirit of inducing-point methods: rather than optimising over the full space of posterior laws, we restrict attention to the family $\mathcal{Q}_\Lambda$ in which only the law of a low-dimensional feature vector $\mbf{u}$ is free to vary whereas the conditional law of $f$ given $\mbf{u}$ is held fixed to its prior value. This reduces the original optimization problem in (\ref{eq:exact-pro-gp-objective}) to one over $\mathcal{P}(\mathbb{R}^m)$ with $m \ll n$. Once again, this result is obtained as a corollary of a more general result, this time for measurable feature maps (see Appendix \ref{ap:theoretical-proofs}).

We note that the inducing point objective is closely related to the PPGPR objective introduced in \cite{jankowiak2020parametric}.
However, whereas PPGPR relies on a Gaussian approximation to $\tau$, we place \emph{no} parametric restriction on this distribution. This additional flexibility is fundamental: it allows PrO-GPs to represent, for example, multimodal, asymmetric, or otherwise highly non-Gaussian distributions over latent functions, and hence to retain multiple distinct explanations of the observed data simultaneously. Consequently, the induced predictive distribution can capture forms of heterogeneity and multimodality inaccessible to a Gaussian mixing law, providing substantially greater capacity to adapt under model misspecification in comparison to PPGPR. See Appendix \ref{ap:ppgr} for a proof that PPGPR is a Gaussian restriction to $\tau$.

\paragraph{Prediction} Conditional on the inducing values $\mbf{u}$, the GP prior is $ \Pi(\dd f\mid \mbf{f}_\Lambda =\mbf{u}) \sim \mathcal{GP}(a_{\mbf{u}},r_\Lambda)$, where 
\begin{align}
a_\mbf{u}(x) &= m(x)
    +
    \mathbf k_{x \Lambda}^\top K_{\Lambda\Lambda}^{-1}
    (\mbf{u} -\mathbf m_{\Lambda}) \notag \\
    r_\Lambda(x, x')
    &=
    k(x, x')
    -
    \mathbf k_{x \Lambda}^\top K_{\Lambda\Lambda}^{-1}\mathbf k_{\Lambda x'}.
    \label{eq:r_lambda}
\end{align}
Hence, for a candidate law $\tau \in \mathcal{P}(\mathbb{R}^m)$
\begin{align*}
    P_{\tau}^{x_\star} 
&=
\int_{\mathbb R^m}
\mathcal N\!\left(
a_{\mbf{u}}(x_\star),
r_\Lambda(x_\star,x_\star)+\sigma^2
\right)
\tau(\mathrm d\mbf{u}) \\
    p_\tau(y \mid x_\star) &=\int_{\mathbb{R}^m} \varphi_{v}(y - a_{\mbf{u}}(x_\star)) \tau( \dd \mbf{u}),
\end{align*}
where $\varphi_{v}$ denotes the density of a centred Gaussian with variance $v = \sigma^2 + r_\Lambda(x_\star, x_\star)$. To obtain the posterior predictive $P_{Q_n^{\text{PrO}}}^{x_\star}$, simply replace the candidate law $\tau$ by the optimised law $\tau_{\Lambda,n}$.
\subsection{Computation}
\label{sec:computation}
Both the exact objective \eqref{eq:exact-pro-gp-objective} and its inducing approximation \eqref{eq:inducing-pro-gp-objective} are optimisation problems over probability measures on the Euclidean spaces, $\mathbb R^n$ and $\mathbb R^m$ respectively.
Both these objectives have the form
\begin{equation}
\label{eq:common-finite-dim-opt}
\mathcal{J}[\tau] \coloneq-\frac{\lambda_n}{n}\sum_{i=1}^n\log\int g_i(\mbf{u})\,\tau(\dd\mbf{u})+\KL(\tau\Vert\pi)
\end{equation}
where $\pi\in\{\pi_X,\pi_\Lambda\}$, and $g_i(\mbf{u})$ is the relevant Gaussian density.
For computation, introduce a collection of $K$ particles $\mathbf{U}^{1:K} =( \mathbf{U}^1, \ldots ,\mathbf{U}^K )$ with $\mathbf{U}^k \in \mathbb{R}^n$ or $\mathbb{R}^m$.
Define
\begin{equation}
    \tau_{n,K}(\dd\mbf{U}^{1:K})
    \propto
    \prod_{k=1}^K\pi(\dd\mbf{U}^k)
    \prod_{i=1}^n
    \Big\{
    \frac{1}{K}\sum_{k=1}^K g_i(\mbf{U}^k)
    \Big\}^{\frac{K\lambda_n}{n}}.
    \label{eq:replica-gibbs-target}
    \vspace{2mm}
\end{equation}
Finally, write $\bar{\tau}_{n,K}(\dd\mbf{u})\coloneq\int \tau_{n,K}(\dd\mbf{u},\dd\mbf{U}^{2:K})$ for the common one-marginal of $\tau_{n,K}$.
Under mild assumptions, $\bar{\tau}_{n,K}\rightarrow\tau_n$ as $K\rightarrow\infty$, where $\tau_n:=\argmin \mathcal{J}[\tau]$ \citep{mclatchie2026markov}.
That is, the marginals of \eqref{eq:replica-gibbs-target} converge to the minimiser of \eqref{eq:common-finite-dim-opt} as $K\to\infty$.
Thus, approximate samples from the minimising law can be obtained by sampling $\tau_{n,K}$, and treating the particles $\mbf{U}^1,\ldots,\mbf{U}^K$ as approximate draws.

Due to its product structure, it is straightforward to sample $\tau_{n,K}$ using a blocked Gibbs sampler.
Suppose now that $r:=K\lambda_n/n\in\mathbb{N}$. 
By the multinomial theorem, each $r$th power can be represented by allocating $r$ copies of observation $i$ among the $K$ particles.
Let $C_{ik}$ denote the number allocated to particle $k$, so $\sum_k C_{ik}=r$.
Conditional on the particles, these counts are independent over observations and satisfy
\begin{equation}
(C_{i1},\ldots,C_{iK})\mid\mbf{U}^{1:K}
\sim
\operatorname{Mult}(r;\rho_{i1},\ldots,\rho_{iK}), \label{eq:replica-allocations}
\end{equation}
where $\smash{\rho_{ik}={g_i(\mbf{U}^k)}/{\sum_{\ell=1}^K g_i(\mbf{U}^{\ell})}}$.
Write $\mbf C=(C_{ik})_{i,k}$ and let $\tau_{n,K,r}$ denote the augmented law of $(\mbf U^{1:K},\mbf C)$, whose marginal over $\mbf U^{1:K}$ is $\tau_{n,K}$.
Conditional on all counts, the particles are independent, with
\begin{equation}
\tau_{n,K,r}(\dd\mbf{U}^k\mid\mbf{C})
\propto
\pi(\dd\mbf{U}^k)
\prod_{i=1}^n g_i(\mbf{U}^k)^{C_{ik}}.
\label{eq:replica-conditionals}
\end{equation}
These conditionals can be sampled exactly.
To see this, write $\pi=\mathcal N(\mbf m_\pi,K_\pi)$ and
$g_i(\mbf u)=\varphi_{v_i}(y_i-b_i-\mbf h_i^\top\mbf u)$.
In the exact construction, $(K_\pi,b_i,\mbf h_i,v_i)=(K_{XX},0,\mbf e_i,\sigma^2)$; in the inducing construction, $K_\pi = K_{\Lambda\Lambda}$, $\mbf h_i=K_{\Lambda\Lambda}^{-1}\mbf k_{\Lambda x_i}$, $b_i=m(x_i)-\mbf h_i^\top\mbf m_\Lambda$, and
$v_i=r_\Lambda(x_i,x_i)+\sigma^2$.
Then \eqref{eq:replica-conditionals} is $\mathcal N(\mbf m_k^\star,\Sigma_k)$, where
\begin{align}
\Sigma_k^{-1}
&=
K_\pi^{-1}
+
\sum_{i=1}^n\frac{C_{ik}}{v_i}\mbf h_i\mbf h_i^\top,
\notag\\[-1pt]
\mbf m_k^\star
&=
\Sigma_k\left\{
K_\pi^{-1}\mbf m_\pi
+
\sum_{i=1}^n
\frac{C_{ik}(y_i-b_i)}{v_i}\mbf h_i
\right\}.
\label{eq:replica-gaussian-update}
\end{align}

A complete sweep alternates the allocation update \eqref{eq:replica-allocations} with the Gaussian particle updates \eqref{eq:replica-gaussian-update}.
The particles are interchangeable copies, or \emph{replicas}, of the feature vector.
We thus refer to this blocked Gibbs procedure as \emph{Replica--Gibbs}.
The allocation step couples the replicas through their relative predictive fit, while the Gaussian step allows different replicas to specialise to different explanations of the data.
After burn-in, the empirical measures from retained sweeps are averaged in the predictive mixture.
The algorithm is exact for \eqref{eq:replica-gibbs-target} and has no step-size or time-discretisation error. The precise algorithm and its computational complexity are provided in Appendix \ref{ap:computational-complexity}. An alternative sampling scheme, based on the Wasserstein Gradient Flow \citep{ambrosio2005gradient}, is presented in Appendix \ref{ap:wgf-sampling}.

\subsection{Hyperparameter Optimisation}
\label{sec:hyperparam-optimisation}

Thus far, we have treated the PrO-GP kernel hyperparameters $\gamma$ as fixed.
This is natural when PrO inference is used to robustify an \emph{existing} GP model: holding $\gamma$ fixed isolates the additional predictive flexibility arising from the PrO posterior, rather than conflating it with changes to the underlying model specification.
Accordingly, we treat $\gamma$ as an external, plug-in quantity. 
It may be specified by expert elicitation, leave-one-out predictive performance, or by maximising the marginal log-likelihood of the underlying GP. 
An alternative approach is to treat PrO-GPs as a standalone model class, freely optimising $\gamma$ by minimising the predictive score in \eqref{eq:common-finite-dim-opt} via gradient steps.
Detailed guidance is provided in Appendix \ref{ap:kernel-hyperparameters}.

The observation variance $\sigma^2$ plays a different role. Consider the PrO predictive variance decomposition at $x_\star$: $
\operatorname{Var}(Y_\star\mid x_\star,\mathcal{D}_n)
    =
    \sigma^2
    +
    s_X(x_\star,x_\star)
    +
    \operatorname{Var}_{\mbf{U}\sim\tau_n}
    \bigl[\mu_{\mbf{U}}(x_\star)\bigr].
$
Unlike in standard GPs, predictive variation can therefore be represented both through the observation model and through the mixing law $\tau_n$. Thus, fixing $\sigma^2$ at an inappropriate value would confound the predictive variance represented by $\tau_n$ and restrict the resulting predictive distribution. Also note that jointly optimising $\sigma^2$ and $\tau$ on the same data would allow the model to interpolate the training points, driving the observation variance to the degenerate solution $\sigma \to 0$ (see Appendix \ref{ap:degenerate-sigma}).
Therefore, we partition the data $\mathcal{D}_n$ into disjoint subsets $\mathcal{D}_n = \mathcal{D}_{\text{train}} \cup \mathcal{D}_{\text{val}}$.
We then alternate between updating $\tau$ according to Section \ref{sec:computation}, with $\sigma^2$ fixed, and updating $\sigma^2$ by minimising the held-out negative log predictive density on $\mathcal{D}_{\text{val}}$, with $\tau$ fixed.

Lastly, we consider the score-weighting term $\lambda_n$. To isolate the difference in predictive performance between PrO-GPs and Bayesian GPs, we fix $\lambda_n / n =1$ so that both posteriors have the same prior-to-likelihood weighting.
See Appendix \ref{ap:lambda_n_discussion} for detailed discussion.

\section{EXPERIMENTS} \label{sec:Experiments}

In this section, we evaluate the empirical performance of the model variants presented in Table \ref{table:acronyms}. Throughout, we evaluate via the mean negative log predictive density (NLPD) on held-out test data (see Appendix \ref{ap:nlpd} for the definitions for \GP{} and \PROGP{}).  First, we analyse the robustness of \PROGP{}s compared to \GP{}s under synthetic misspecification in Section \ref{sec:synthetic}. Second, to demonstrate the practical benefits of \PROGP{}s, we apply them to the task of probabilistic interpolation of traffic flow on a road network in Section \ref{sec:traffic-flow}. Lastly, we perform an extensive evaluation of \PROGP{}s on UCI regression benchmarks in Section \ref{sec:uci-regression}.

\begin{table}[h!]
\vspace{-0.5em}
\caption{Model variants}
\vspace{0.5em}
\centering
\begin{tabular}{lp{0.68\columnwidth}}
\hline
\textbf{Model} & \textbf{Description} \\
\hline
\GP{}    & Bayesian GP fit with marginal log-likelihood maximisation \\
\PROGP & PrO-GP with $\sigma^2$ estimated as in Section \ref{sec:hyperparam-optimisation}. $\gamma$ is identical to the \GP{}. \\
\SVGP & Sparse Variational GP. \\
\PPGPR & \SVGP{} optimised with PPGPR. \\
\IPROGP\textsuperscript{\scriptsize{S}/\scriptsize{P}} &\PROGP{} with inducing point approximation. Inducing point locations and $\gamma$ are identical to the ({\scriptsize{S}})  \SVGP{} or the ({\scriptsize{P}}) \PPGPR{} model. \\
\hline
\end{tabular}
\setlength{\belowcaptionskip}{2pt}
\label{table:acronyms}
\vspace{-0.5em}
\end{table}

\begin{figure*}[h!]
    \centering
    \includegraphics[width=1.0\textwidth]{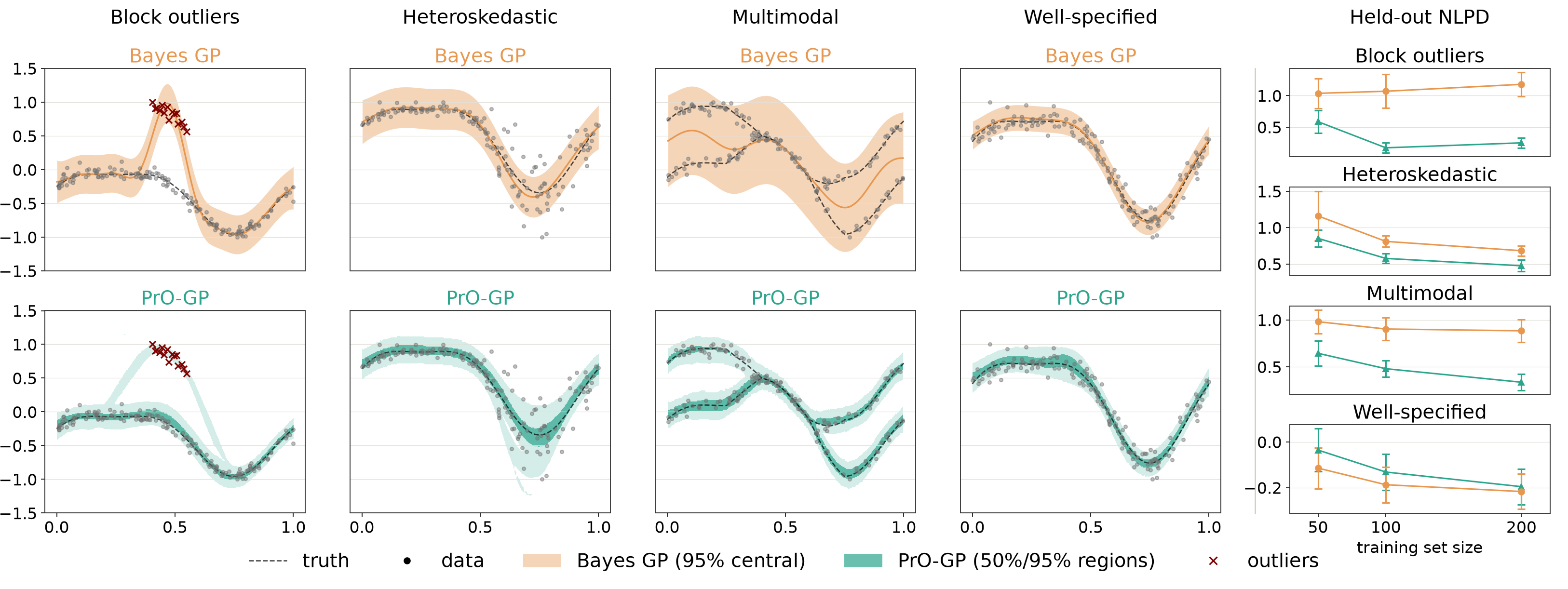}
    \vspace*{-6mm}
    \caption{\textit{Left}: Illustrative \GP{} and \PROGP{} predictive densities on four synthetic regression tasks. \textit{Right:} Mean NLPD (lower is better) and 95\% credible interval over increasing training set size, randomised over twenty DGPs.}
    \label{fig:synthetic-results-1}
\end{figure*}

\subsection{Synthetic misspecification}
\label{sec:synthetic}

The form of the \PROGP{} posterior predictive in \eqref{eq:pro-predictive} suggests it will outperform the Bayesian \GP{} posterior whenever a single Gaussian, conditional on the latent function mean, is insufficient to describe the predictive distribution. In Figure \ref{fig:synthetic-results-1}, we verify this observation, showing that the \PROGP{} predictive yields lower mean NLPD under block outliers, heteroskedasticity, and multimodality. The reader is referred to Appendix \ref{ap:synthetic} for the exact experimental setup. We describe our observations in each case below.

\textit{Block outliers:} \GP{}s sharply fit the outliers, causing underestimation of the ground truth and becoming increasingly overconfident as dataset size increases. In contrast, \PROGP{}s represent a mixture of scenarios, whereby the latent function may or may not follow the trend implied by the outlying observations. \textit{Heteroskedastic:} \GP{}s inflate observation variance across the whole domain to account for heteroskedasticity, whereas \PROGP{}s can leverage additional functional uncertainty to represent the heterogeneous regions. \textit{Multimodal:} \GP{}s compromise by placing most of the predictive density between the bifurcation, describing neither mode accurately.  \PROGP{}s can represent this uncertainty directly via $\tau$. \textit{Well-specified:} as dataset size increases, the performance of \PROGP{}s converges to \GP{}s in the well-specified case, wherein a single, correctly specified function describes the DGP. If the model is well-specified a standard \GP{} is preferable, but identifying the precise form of misspecification and correcting it can be challenging in practice.

\subsection{Probabilistic interpolation of traffic flow}
\label{sec:traffic-flow}

We demonstrate the benefits of \PROGP{}s as a robustification tool for regression problems where defining a well-specified prior is intrinsically challenging. We use a dataset of traffic flow on the San Jose freeway sensor network \citep{chen2001freeway}, following the setup of \citet{borovitskiy2021matern} (see Appendix \ref{ap:real-examples-1}), who introduce a graph Mat\'{e}rn kernel to parameterise a \GP{} prior on the space of undirected graphs, resulting in graph \GP{}s. We fit graph \GP{}s and graph \PROGP{}s, using an increasing number of sensors for training. The results are presented in Table \ref{tab:graph-gp} and visualisation of the model predictive standard deviations overlaid onto the road network is provided in Figure \ref{fig:road-std}.

\begin{table}[h!]
\centering
\setlength{\tabcolsep}{6pt}
\renewcommand{\arraystretch}{0.95}
\caption{Graph regression NLPD ($\pm 1$ SE).}
\label{tab:graph-gp}
\begin{tabular}{@{}rccc@{}}
\toprule
Sensors & Graph \GP{} & Graph \PROGP{} & $\Delta_{\text{NLPD}}$ \\
\midrule
200 & 1.13 $\pm$ 0.02 & 1.00 $\pm$ 0.02 & 0.13 \\
225 & 1.13 $\pm$ 0.02 & 0.98 $\pm$ 0.03 & 0.15 \\
250 & 1.12 $\pm$ 0.03 & 0.93 $\pm$ 0.04 & 0.19 \\
275 & 1.11 $\pm$ 0.03 & 0.89 $\pm$ 0.04 & 0.22 \\
\bottomrule
\end{tabular}
\end{table}

Graph \PROGP{} substantially improves mean NLPD and yields better calibrated predictive distributions compared to Graph \GP{}. Furthermore, the difference in mean NLPD increases as the number of training sensors increases. This suggests the graph \GP{} is non-trivially misspecified and is qualitatively consistent with finite-dimensional PrO theory: as \(n\) grows, PrO posteriors approach the predictively optimal model average, whereas the Bayesian posterior concentrates around the single best-fitting function and predictive performance plateaus. 

Interestingly, whilst the predictive means are almost identical (see Figure \ref{fig:road_grid}, Appendix \ref{ap:real-examples-1}), observe in Figure \ref{fig:road-std} that the predictive standard deviations are completely different. Apart from locations with sparse measurements (e.g. bottom right), Graph \GP{}'s predictive standard deviation is uniform. In contrast, Graph \PROGP{} predicts significantly more variability, even in regions with more training data. We hypothesise this reflects the underlying variability in the traffic and congestion patterns in different roads, which Graph \GP{} cannot represent. In accordance with \citet{liu2025detecting}, we find that PrO posteriors are a useful tool for both detecting and mitigating misspecification.

\begin{figure}[h!]
\centering
\includegraphics[width=0.5\textwidth]{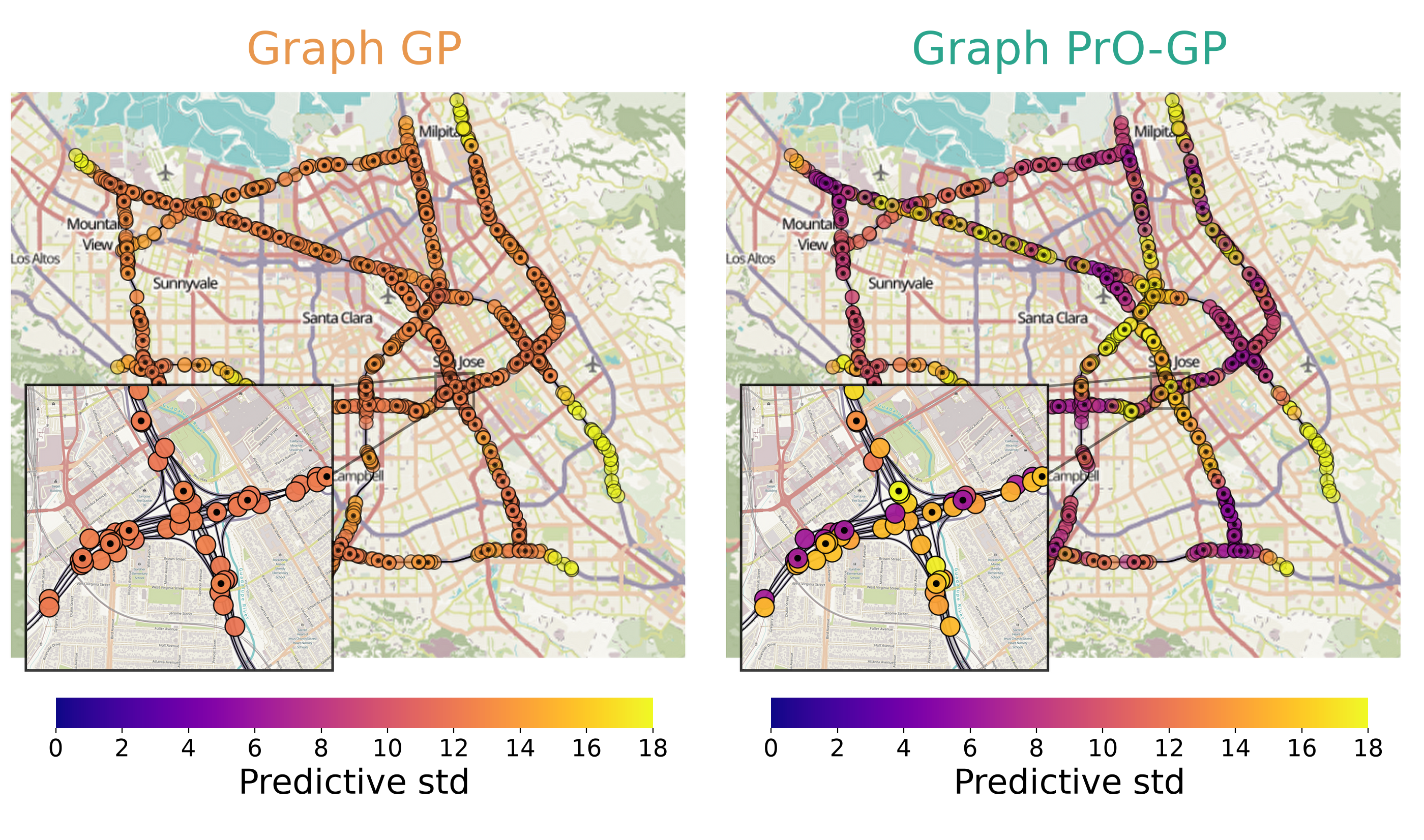}
    \vspace{-7mm}
    \caption{Graph \GP{} and Graph \PROGP{} posterior predictive standard deviations for the road network. Circles are sensor nodes, black dots are training nodes. Zoomed in predictions are in the bottom left.}
    \label{fig:road-std}
\end{figure}

\begin{table*}[t]
\vspace{-2mm}
\centering
\small
\setlength{\tabcolsep}{2.0pt}
\renewcommand{\arraystretch}{1.0}
\caption{Mean NLPD ($\pm 1$ SE) across 16 UCI regression datasets with 10-fold random splits. \textit{Exact}: \GP{} and \PROGP{} evaluated on datasets with no inducing point approximation. \textit{Sparse}: \SVGP{}, \PPGPR{}, \SIPROGP{} and \PIPROGP{} evaluated on datasets with inducing point approximation. Recall here that \SIPROGP{} and \PIPROGP{} are inducing point approximations of \PROGP{} with kernel-hyperparameters and inducing point locations given by those of \SVGP{} and \PPGPR{}, respectively. See~Table \ref{table:acronyms} for full details. }
\vspace{0.5em}
\label{table:uci-results}

\begin{tabular}{@{}clcccccccc@{}}
\toprule

\multirow{3}{*}{\rotatebox[origin=c]{90}{\textit{Exact}}}
& 
& \shortstack{concreteslump\\{\scriptsize $n=103$}}
& \shortstack{autos\\{\scriptsize $n=159$}}
& \shortstack{machine\\{\scriptsize $n=209$}}
& \shortstack{autompg\\{\scriptsize $n=392$}}
& \shortstack{housing\\{\scriptsize $n=506$}}
& \shortstack{stock\\{\scriptsize $n=536$}}
& \shortstack{energy\\{\scriptsize $n=768$}} 
& \shortstack{concrete\\{\scriptsize $n=1030$}} \\
\cmidrule(l){2-10}

& \GP{}
& -1.30 (0.17)
& 0.50 (0.10)
& \textbf{0.61 (0.05)}
& 0.44 (0.07)
& 0.33 (0.06)
& \textbf{0.28 (0.03)}
& -1.68 (0.01)
& 0.18 (0.03)
 \\

& \PROGP{}
& \textbf{-1.65 (0.08)}
& \textbf{0.41 (0.07)}
& \textbf{0.61 (0.05)}
& \textbf{0.30 (0.06)}
& \textbf{0.25 (0.06)}
& \textbf{0.28 (0.03)}
& \textbf{-1.85 (0.02)}
& \textbf{0.03 (0.03)} \\

\midrule

\multirow{5}{*}{\rotatebox[origin=c]{90}{\textit{Sparse}}}
& 
& \shortstack{wine\\{\scriptsize $n=1599$}}
& \shortstack{skillcraft\\{\scriptsize $n=3338$}}
& \shortstack{abalone\\{\scriptsize $n=4177$}}
& \shortstack{parkinsons\\{\scriptsize $n=5875$}}
& \shortstack{airquality\\{\scriptsize $n=6941$}}
& \shortstack{elevators\\{\scriptsize $n=16599$}}
& \shortstack{kin40k\\{\scriptsize $n=40000$}}
& \shortstack{protein\\{\scriptsize $n=45730$}} \\
\cmidrule(l){2-10}

& \SVGP{}
& 0.61 (0.02)
& 0.98 (0.01)
& 1.00 (0.02)
& -2.01 (0.01)
& 0.09 (0.02)
& 0.45 (0.00)
& -0.21 (0.00)
& 1.06 (0.00) \\

& \SIPROGP{}
& \textbf{0.51 (0.03)}
& \textbf{0.95 (0.01)}
& \textbf{0.81 (0.02)}
& \textbf{-2.63 (0.01)}
& -0.28 (0.02)
& \textbf{0.32 (0.00)}
& -0.61 (0.00)
& \textbf{0.64 (0.00)} \\

& \PPGPR{}
& 0.59 (0.03)
& 1.07 (0.02)
& 0.93 (0.03)
& -2.24 (0.03)
& -0.26 (0.03)
& 0.35 (0.01)
& -0.59 (0.00)
& 0.81 (0.00) \\

& \PIPROGP{}
& \textbf{0.51 (0.03)}
& 1.01 (0.02)
& 0.82 (0.02)
& -2.52 (0.02)
& \textbf{-0.31 (0.02)}
& \textbf{0.32 (0.00)}
& \textbf{-0.72 (0.00)}
& 0.65 (0.00) \\
\bottomrule
\end{tabular}
\end{table*}

\subsection{UCI regression}
\label{sec:uci-regression}
\vspace{-1mm}
We finally evaluate \PROGP{}s on a collection of regression datasets from the UCI repository \citep{kelly2026uci}.
In particular, we compare the predictive performance of \PROGP{} and its inducing variants \SIPROGP{} and \PIPROGP{} (see Table~\ref{table:acronyms}), against \GP{}, \SVGP{} \citep{titsias2009variational, hensman2013gaussian} and \PPGPR{} \citep{jankowiak2020parametric}. 
All models use an RBF kernel with independent lengthscales per dimension; see Appendix \ref{ap:uci-regression} for full experimental details. 
The results are presented in Table \ref{table:uci-results}. 

\textit{Exact:} \PROGP{}s consistently reduce the mean NLPD compared to \GP{}s. This is unsurprising: applying a generic kernel and likelihood uniformly across diverse datasets is likely to introduce model misspecification. For certain datasets (\textit{machine} and \textit{stock}), we find that the NLPD for the \GP{} and the \PROGP{} is comparable. Further investigation reveals that in these instances the base \GP{} is reasonably well-specified, and thus the \PROGP{} is not expected to offer an advantage; see Appendix \ref{ap:uci-regression} for supporting analysis. \textit{Sparse}: \IPROGP{}s reduce mean NLPD across all of the considered datasets, regardless of whether hyperparameters were obtained from maximising the \SVGP{} or the \PPGPR{} objective. The consistent performance gap between \PPGPR{} and \PIPROGP{} is particularly notable, suggesting that the non-Gaussianity of the measure $\tau$ is genuinely beneficial for improving predictive calibration.

We conclude that the PrO inferential procedure itself, rather than hyperparameter configuration, can substantially improve predictive uncertainty estimation, particularly in datasets with challenging structure (e.g. \textit{parkinsons, protein}). Although enriching the kernel or likelihood can also help in this regard, \PROGP{}s provide a complementary approach: they retain the underlying GP model class, but allow greater flexibility to adapt to model misspecification compared to the Bayesian posterior. Richer model specifications and PrO inference are therefore not mutually exclusive.

Our primary comparisons above use single-GP baselines, which allow us to isolate the effect of PrO inference, whilst holding the underlying model class fixed.
Explicit mixtures of GPs provide a related but conceptually distinct route to richer predictive distributions by enlarging the statistical model itself; see Section~\ref{sec:exact-feature-reduction}.
For completeness, we include an additional comparison against such models in Appendix~\ref{ap:mixture-comparison}.

\section{CONCLUSION}

Model misspecification in GPs can result in significantly miscalibrated predictive distributions. We introduce PrO-GPs as a natural solution to this problem. We develop the necessary theoretical and computational machinery to make inference for PrO-GPs tractable, and demonstrate through extensive experiments that they improve predictive performance under model misspecification.
Naturally, this improved robustness comes at the cost of additional computational effort. Crucially, however, no additional \emph{modelling} effort is required. Thus, inference remains interpretable with respect to the underlying model. When a practitioner suspects their model is misspecified, therefore, a PrO-GP is a natural approach to robustification.

There are several interesting directions for future work. Readily, we could explore the practical utility of PrO-GPs under alternative scoring rules (see Appendix \ref{ap:theoretical-proofs}) and observation models, such as the logistic likelihood for GP classification \citep{williams1998bayesian}. 
More broadly, it would be of interest to extend this formulation to Multi-Output GPs \citep{bonilla2007multi}, Deep GPs \citep{damianou2013deep} and Spatio-Temporal GPs \citep{sarkka2013spatiotemporal}, to name just a few.

\newpage

\subsection*{AI use statement}

In this work, we used generative AI tools to generate synthetic data sets, assist in the writing of proofs, implement methods, and clean and reformat datasets. We have not used generative AI tools to help develop theoretical models or conceptual frameworks, formulate mathematical claims, provide critical ingredients for proving mathematical claims, propose or refine hypotheses, design or provide feedback on research methodology or experiments, support qualitative and thematic data analysis, or interpret results. Assistance with translation is not applicable to this work. We have reviewed all AI-assisted work. We take responsibility for the final content of this work, including text, claims, or artifacts produced with the aid of generative AI.

\subsection*{Acknowledgements}
CL was funded by the UK Engineering and Physical Sciences Research Council (EPSRC) grant
EP/Y034767/1.


\bibliography{references}

\clearpage
\appendix
\thispagestyle{empty}

\onecolumn
\aistatstitle{Supplementary Materials}

\section{NOTATION}
\label{ap:notation}

\vspace{-10mm}
\printunsrtglossary[type=symbols, title={}]

\newpage

\section{THEORETICAL PROOFS}
\label{ap:theoretical-proofs}

In this section we prove Theorems \ref{thm:exact-gp-reduction} and \ref{thm:inducing-point-gp-restriction}. As mentioned in Section \ref{sec:Method}, these are in fact simple corollaries of much more general results for nonparametric models which satisfy a certain feature representation assumption. Hence, we proceed by introducing the setting and theorems in full generality, before proving them in Appendices \ref{ap:proof-theorem-1} and \ref{ap:proof-theorem-2}, and specialising to GP regression under the log-score in Appendix \ref{ap:proof-gps}.

\subsection{Feature reduction for nonparametric models} \label{ap:feature-reduction}

Henceforth, let $\mathcal F$ be a Polish function space and let $\Pi\in\mathcal P(\mathcal F)$ denote a prior measure over functions. For each $f\in\mathcal F$ and $x\in\mathcal X$, let $P_f^x\in\mathcal P(\mathcal Y)$ be the predictive law of the model. Furthermore, let $S: \mathcal{P}(\mathcal{Y}) \times \mathcal{Y} \to \mathbb{R} \cup \{\infty \}$ be a proper scoring rule \citep{gneiting2007strictly}. Define the function-space PrO objective in full generality

\begin{equation}
Q_n^{\mathrm{PrO}}
\in
\argmin_{Q\in\mathcal P(\mathcal F)} \mathcal{J}_{S}^\Pi[Q], \qquad \mathcal{J}_{S}^\Pi[Q] \coloneq
\frac{\lambda_n}{n}
\sum_{i=1}^n
S(P_Q^{x_i},y_i)
+
\KL(Q\Vert\Pi).
\label{eq:full-function-pro}
\end{equation}

We state the following assumption in which the  predictive law only depends on $f \in \mathcal{F}$ through a feature map.

\begin{assumption}[Feature factorisation]
\label{as:feature-factorisation}
There exists $d<\infty$, a measurable feature map
$\phi:\mathcal F\to\mathbb R^d$, and, for each $i=1,\ldots,n$, a family of predictive laws
$\{\bar P_{\mbf{u}}^{x_i}:\mbf{u}\in\mathbb R^d\}\subseteq\mathcal P(\mathcal Y)$ such that, for every $f\in\mathcal F$,
\[
P_f^{x_i}
=
\bar P_{\phi(f)}^{x_i},
\qquad
i=1,\ldots,n
\]
\end{assumption}

Under this assumption, we provide the following result.

\begin{theorem}[Exact feature reduction]
\label{thm:exact-feature}
Suppose Assumption \ref{as:feature-factorisation} holds and define
\(
\pi\coloneqq\phi_\#\Pi.
\)
For any $\tau\in\mathcal P(\mathbb R^d)$ and $i=1,\ldots,n$ let
\[
\bar P_\tau^{x_i}
\coloneqq
\int_{\mathbb R^d}
\bar P_{\mbf{u}}^{x_i}\,
\tau(\mathrm d\mbf{u}).
\]
Then minimising the function-space PrO objective in
\eqref{eq:full-function-pro} over
$Q\in\mathcal P(\mathcal F)$ is equivalent to
\begin{equation}
    \tau_n\in
    \argmin_{\tau\in\mathcal P(\mathbb R^d)}
    \left\{
        \frac{\lambda_n}{n}
        \sum_{i=1}^n S(\bar P_\tau^{x_i},y_i)
        +
        \text{KL}(\tau\Vert\pi)
    \right\}.
    \label{eq:fs-opt}
\end{equation}
For every minimiser $\tau_n$ with
$\text{KL}(\tau_n\Vert\pi)<\infty$, the corresponding
function-space minimiser is
\begin{equation*}
    Q_n^{\text{PrO}}(\mathrm df)
    =
    \int_{\mathbb R^d}
    \Pi(\mathrm df\mid \mbf{u})\,
    \tau_n(\mathrm d \mbf{u}).
\end{equation*}
Conversely, every finite-objective function-space minimiser is of
this form.
\end{theorem}

Next, we state a theorem regarding the inducing feature approximation.

\begin{theorem}[Inducing feature restriction]
\label{thm:inducing}
Let $L:\mathcal F\to\mathbb R^m$ be a measurable feature map, write
$\mbf{u}=Lf$, and let
$
\pi_L=L_\#\Pi
$
denote the induced prior law of $\mbf{u}$.
For $\tau\in\mathcal P(\mathbb R^m)$ with
$\KL(\tau\Vert\pi_L)<\infty$, define
\begin{equation*}
Q^\tau(\mathrm df)
\coloneq
\int_{\mathbb R^m}
\Pi(\mathrm df\mid\mbf{u})\,
\tau(\mathrm d\mbf{u}),
\end{equation*}
and for $i=1,\ldots,n$
\begin{equation*}
\widetilde P_{\tau}^{x_i}
\coloneq
\int_{\mathbb R^m}
\left[\int_{\mathcal F}
P_f^{x_i}\,
\Pi(\mathrm df\mid\mbf{u}) \right]
\tau(\mathrm d\mbf{u}).
\label{eq:inducing-predictive}
\end{equation*}
Then minimising the function-space PrO objective over the restricted family
\[
\mathcal Q_L
\coloneq
\left\{
Q^\tau:
\tau\in\mathcal P(\mathbb R^m),
\ \KL(\tau\Vert\pi_L)<\infty
\right\}
\]
is equivalent to
\begin{equation*}
\tau_{L,n}
\in
\argmin_{\tau\in\mathcal P(\mathbb R^m)}
\left\{
\frac{\lambda_n}{n}
\sum_{i=1}^n
S(\widetilde P_\tau^{x_i},y_i)
+
\KL(\tau\Vert\pi_L)
\right\}.
\label{eq:inducing-general-objective}
\end{equation*}
\end{theorem}

\subsection{Proof of Theorem \ref{thm:exact-feature}} \label{ap:proof-theorem-1}

First, we state and prove two intermediate lemmas.

\begin{lemma} \label{thm:int-prop}
Let $T: \mathcal{F} \to \mathbb{R}^d$ be measurable. Define
\begin{align*}
    \tau_Q \coloneq T_\#Q, \qquad \pi_T \coloneq T_\# \Pi
\end{align*}
and let 
\begin{align*}
    Q_{\mbfu} \coloneq Q(\mathrm{d}f \mid \mbfu), \qquad\Pi_\mbfu \coloneq \Pi(\mathrm{d}f \mid \mbfu)
\end{align*}
 denote the regular conditional laws of $f$ given $\mbfu = T(f)$ under $Q$ and $\Pi$ respectively. Then
\begin{equation*}
    \text{KL}(Q \Vert \Pi) =\text{KL}(\tau_Q \Vert \pi_T) + \int_{\mathbb{R}^d} \text{KL}(Q_\mbfu \Vert \Pi_\mbfu) \tau_Q(\mathrm{d}\mbfu)
\label{eq:chain-rule}
\end{equation*}
\end{lemma}

\begin{proof}
Consider the two cases: $Q \not\ll \Pi$ and $Q \ll \Pi$.

Suppose first that $Q \not\ll \Pi$. Then $\text{KL}(Q \Vert \Pi) = + \infty$. If $\tau_Q \not\ll \pi_T$, then $\text{KL}(\tau_Q \Vert \pi_T) = +\infty$ and the RHS is also $+\infty$. Otherwise, $\tau_Q \ll \pi_T$ and $Q \not\ll \Pi$. In this case, there is a set of $\mbf{u}$-values of positive $\tau_Q$-measure on which $Q_\mbfu \not\ll \Pi_\mbfu$; otherwise, the disintegrations of $Q$ and $\Pi$, together with $\tau_Q \ll \pi_T$, would imply $Q \ll \Pi$, a contradiction. Hence, the conditional KL integral is $+\infty$. In either case, the right-hand side is also $+\infty$.

Now suppose $Q \ll \Pi$. Then, by Radon-Nikodym, $h(f) \coloneq \frac{\mathrm{d} Q}{\mathrm{d} \Pi}(f)$ exists. Also, $h \geq 0$ and $\int h \mathrm{d} \Pi = 1$. Define 
\begin{equation*}
    a(\mbfu) \coloneq \mathbb{E}_{\Pi}[h(F) \mid T(F)=\mbfu]
\end{equation*}
We claim that $a(\mbfu) = \frac{\mathrm{d} \tau_Q}{ \mathrm{d}\pi_T}(\mbfu)$ for $\pi_T$-almost every $\mbfu$. To see this, observe that for every bounded measurable test function $\psi : \mathbb{R}^d \to \mathbb{R}$, we have
\begin{equation*}
    \int_{\mathbb{R}^d} \psi(u) \tau_Q(\mathrm{d}\mbfu) = \int_{\mathcal{F}} \psi(T(f)) Q(\mathrm{d}f) = \int_{\mathcal{F}} \psi(T(f)) h(f) \Pi(\mathrm{d}f) = \int_{\mathbb{R}^d} \psi(\mbfu) a(\mbfu) \pi_T(\mathrm{d}\mbfu).
\end{equation*}
The first equality follows from the definition of $\tau_Q$ and pushforward change of variables, and the last from the law of total expectation and the fact that $\pi_T = T_\#\Pi$. Thus the claim holds. Now define $b_\mbfu(f) \coloneq \frac{h(f)}{a(\mbfu)}$ whenever $a(\mbfu) > 0$. We claim that $b_\mbfu$ is the Radon-Nikodym derivative of $Q_\mbfu$ with respect to $\Pi_\mbfu$. To verify this, let $\varphi: \mathcal{F} \to \mathbb{R}$ be bounded and measurable. By disintegration under $\Pi$,
\begin{align*}
    \int_{\mathcal{F}} \varphi(f) Q(\mathrm{d}f) &= \int_{\mathcal{F}} \varphi(f) h(f) \Pi(\mathrm{d}f) \\
    &= \int_{\mathbb{R}^d} \left[ \int_{\mathcal{F}} \varphi(f) h(f) \Pi_\mbfu(\mathrm{d}f) \right] \pi_T( \mathrm{d}\mbfu) \\
    &= \int_{\mathbb{R}^d} \left[ \int_{\mathcal{F}} \varphi(f) a(\mbfu) b_\mbfu(f)\Pi_\mbfu(\mathrm{d}f) \right] \pi_T( \mathrm{d}\mbfu) \\
&= \int_{\mathbb{R}^d} \left[ \int_{\mathcal{F}} \varphi(f) b_\mbfu(f)\Pi_\mbfu(\mathrm{d}f) \right] \tau_Q( \mathrm{d}\mbfu) \\
\end{align*}
 Similarly, disintegration of $Q$ gives $\int_{\mathcal{F}} \varphi(f) Q(\mathrm{d}f) = \int_{\mathbb{R}^d} \left[\int_{\mathcal{F}} \varphi(f) Q_\mbfu(\dd f)  \right] \tau_Q(\dd \mbfu)$. Hence, $Q_\mbfu(\dd f) = b_\mbfu(f)\Pi_\mbfu(\dd f)$ for $\tau_Q$-almost every $\mbfu$. It remains to compute
\begin{align*}
    \text{KL}(Q\Vert\Pi)
    &= \int_{\mathcal F} h(f)\log h(f)\,\Pi(\dd f) \\
    &= \int_{\mathbb R^d}\int_{\mathcal F} a(\mbfu)b_\mbfu(f)\log\!\left(a(\mbfu)b_\mbfu(f)\right)\Pi_\mbfu(\dd f)\pi_T(\dd\mbfu) \\
    &= \int_{\mathbb R^d} a(\mbfu)\log a(\mbfu)\,\pi_T(\dd\mbfu)
    + \int_{\mathbb R^d}\left[\int_{\mathcal F} b_\mbfu(f)\log b_\mbfu(f)\,\Pi_\mbfu(\dd f)\right]\tau_Q(\dd\mbfu) \\
    &= \text{KL}(\tau_Q\Vert\pi_T)
    + \int_{\mathbb R^d}\text{KL}(Q_\mbfu\Vert\Pi_\mbfu)\,\tau_Q(\dd\mbfu).
\end{align*}
\end{proof}

\begin{lemma}
\label{thm:prior-conditional}
    Let $T:\mathcal F\to\mathbb R^d$ be measurable, define $\pi_T\coloneq T_\#\Pi$, and let $\tau\in\mathcal P(\mathbb R^d)$ satisfy $\KL(\tau\Vert\pi_T)<\infty$.
    Define the prior-conditional lift 
\begin{align*}
    Q^{\tau} \coloneq \int_{\mathbb{R}^d} \Pi(\dd f \mid \mbf{u}) \tau(\dd \mbf{u}).
\end{align*} 
Then $T_\# Q^{\tau} = \tau$ and $Q^\tau(\dd f \mid \mbf{u}) = \Pi(\dd f \mid \mbf{u})$, $\tau$-a.e. Consequently, $\text{KL}(Q^\tau \Vert \Pi) = \text{KL}(\tau\Vert\pi_T)$.
\end{lemma}
\begin{proof}
For any Borel $B \subseteq \mathbb{R}^d$,
\begin{align*}
    (T_\#Q^\tau)(B) &= Q^\tau(T^{-1}(B)) = \int \Pi(T^{-1}(B) \mid \mbf{u}) \tau(\dd \mbf{u}) = \int \mathbf{1}_B(\mbf{u}) \tau(\dd \mbf{u}) = \tau(B),
\end{align*}
since $\Pi(\cdot \mid \mbf{u})$ is a regular conditional law given $T(f) = \mbf{u}$, and hence is concentrated on $T^{-1}(\{\mbf{u} \})$. Thus $T_\#Q^\tau=\tau$.

Moreover, for measurable $A \subseteq \mathcal{F}$ and $B \subseteq \mathbb{R}^d$,
\begin{align*}
    Q^{\tau}(A \cap T^{-1}(B)) &= \int_{\mathbb{R}^d} \Pi(A  \cap T^{-1}(B) \mid \mbf{u}) \tau(\dd \mbf{u}) = \int_B \Pi(A \mid \mbf{u}) \tau( \dd \mbf{u}) = \int_{B} \Pi(A \mid \mbf{u}) (T_\# Q^{\tau})(\dd \mbf{u}).
\end{align*}
Therefore $\Pi(\cdot \mid \mbf{u})$ is a regular conditional probability under $Q^\tau$, and $Q^{\tau}(\dd f \mid \mbfu) = \Pi(\dd f \mid \mbfu)$, $\tau$-a.e. Applying Lemma \ref{thm:int-prop} with $Q=Q^{\tau}$,
\begin{align*}
    \text{KL}(Q^{\tau}\Vert\Pi)=\text{KL}(\tau\Vert\pi_T),
\end{align*}
since the conditional KL term vanishes.
\end{proof}

Now, let $\tau_Q \coloneq \phi_{\#}Q$ and $\pi \coloneq \phi_\# \Pi$. By Assumption \ref{as:feature-factorisation}, for $i = 1, \ldots,n$ 
\begin{equation*}
    P_Q^{x_i} = \int_{\mathcal{F}} P_f^{x_i} Q(\dd f) = \int_{\mathcal{F}} \bar{P}_{\phi(f)}^{x_i} Q(\dd f) = \int \bar{P}_u^{x_i} \tau_Q(\dd \mbf{u}) = \bar{P}_{\tau_Q}^{x_i}
\end{equation*}
Recall the full PrO functional $\mathcal{J}_{S}^{\Pi}[Q]$ from \eqref{eq:full-function-pro}, and define, for any $\tau \in \mathcal{P}(\mathbb{R}^d)$, the reduced functional $\bar{\mathcal{J}}_{S}^{\pi}[\tau]$
\begin{align*}
    \bar{\mathcal{J}}_{S}^{\pi}[\tau] &\coloneq \frac{\lambda_n}{n} \sum_{i=1}^n  S(\bar{P}_\tau^{x_i}, y_i) + \text{KL}(\tau \Vert\pi).
\end{align*}
Applying Lemma \ref{thm:int-prop} with $T = \phi$,
\begin{align*}
    \mathcal{J}_{S}^{\Pi}[Q] &=  \bar{\mathcal{J}}_{S}^{\pi}[\tau_Q] + \int_{\mathbb{R}^d}\text{KL}(Q_\mbfu\Vert\Pi_\mbfu)\tau_Q(\dd\mbfu) \geq  \bar{\mathcal{J}}_{S}^{\pi}[\tau_Q] \\
    \implies \inf_{Q\in\mathcal{P}(\mathcal{F})}\mathcal{J}_{S}^{\Pi}[Q] &\geq \inf_{\tau\in\mathcal{P}(\mathbb{R}^d)} \bar{\mathcal{J}}_{S}^{\pi}[\tau].
\end{align*}
For the reverse inequality, fix $\tau \in \mathcal{P}(\mathbb{R}^d)$ and assume $\text{KL}(\tau \Vert \pi) < \infty$. Let $Q^\tau$ denote the prior conditional lift from Lemma \ref{thm:prior-conditional}. Then $\text{KL}(Q^\tau \Vert \Pi) = \text{KL}(\tau \Vert \pi)$. Furthermore, Assumption \ref{as:feature-factorisation} gives $P_{Q^\tau}^{x_i} = \bar{P}_{\tau}^{x_i}$ for $i = 1, \ldots,n$. Hence $\mathcal{J}_{S}^{\Pi}[Q^\tau] =  \bar{\mathcal{J}}_{S}^{\pi}[\tau]$ and it follows that
\begin{equation*}
    \inf_{Q\in \mathcal{P}(\mathcal{F})} \mathcal{J}_{S}^{\Pi}[Q] \leq \inf_{\tau \in \mathcal{P}(\mathbb{R}^d)} \bar{\mathcal{J}}_{S}^{\pi}[\tau]. 
\end{equation*}
Combining the two inequalities proves
\eqref{eq:fs-opt}.

Now, let $\tau_n$ be a minimiser of \eqref{eq:fs-opt} with $\text{KL}(\tau_n \Vert \pi) < \infty$. Its prior-conditional lift $Q_n^{\text{PrO}} = Q^{\tau_n}$ satisfies $ \mathcal{J}_{S}^{\Pi}[Q_n^{\text{PrO}}] = \bar{\mathcal{J}}_{S}^{\pi}[\tau_n]$ and hence minimises the original function-space objective. Conversely, let $Q_n$ be a function-space minimiser and set $\tau_n = \phi_\#Q_n$. Lemma \ref{thm:int-prop} gives $ \mathcal{J}_{S}^{\Pi}[Q_n] = \bar{\mathcal{J}}_{S}^{\pi}[\tau_n] + \int_{\mathbb{R}^d}\text{KL}((Q_n)_\mathbf{u} \Vert \Pi_{\mbfu}) \tau_n(\dd \mbfu)$. However, minimality of $Q_n$ forces $\int_{\mathbb{R}^d}\text{KL}((Q_n)_\mathbf{u} \Vert \Pi_{\mbfu}) \tau_n(\dd \mbfu) = 0$, in which case $(Q_n)_{\mbfu} = \Pi_\mbfu$ for $\tau_n$-a.e. $\mbfu$. Consequently, disintegration gives $Q_n(\dd f) = \int_{\mathbb{R}^d} \Pi(\dd f \mid \mbfu) \tau_n(\dd \mbfu)$.

\subsection{Proof of Theorem 4}
\label{ap:proof-theorem-2}
Fix $\tau\in\mathcal P(\mathbb R^m)$ satisfying $\KL(\tau\Vert\pi_L)<\infty$. Applying Lemma \ref{thm:prior-conditional} with $T = L$ gives $\text{KL}(Q^\tau \Vert \Pi) = \text{KL}(\tau \Vert \pi_L)$.
Moreover, for each $i=1,\ldots,n$, 
\begin{align*}
    P_{Q^{\tau}}^{x_i} = \int_{\mathcal{F}} P_f^{x_i} Q^{\tau}(\dd f) = \int_{\mathbb{R}^m} \left[ \int_\mathcal{F} P_f^{x_i} \Pi(\dd f \mid \mbf{u})\right] \tau(\dd \mbf{u}) = \widetilde{P}_{\tau}^{x_i}
\end{align*}
Combining these identities yields
\begin{align*}
     \mathcal{J}_{S}^{\Pi}[Q^\tau] = \frac{\lambda_n}{n} \sum_{i=1}^n S(\widetilde{P}_\tau^{x_i}, y_i) + \text{KL}(\tau \Vert \pi_L)
\end{align*}
Since every element of the restricted family $\mathcal{Q}_L$ is of the form $Q^\tau$ for some $\tau \in \mathcal{P}(\mathbb{R}^m)$, minimising the function-space PrO objective over $\mathcal{Q}_L$ is therefore equivalent to
\begin{align*}
\tau_{L,n} \in \argmin_{\tau\in\mathcal P(\mathbb R^m)}\left\{\frac{\lambda_n}{n}\sum_{i=1}^n S(\widetilde P_\tau^{x_i},y_i)
+
\KL(\tau\Vert\pi_L)
\right\}.
\end{align*}

\subsection{Proofs of Theorems 1 \& 2}
\label{ap:proof-gps}
The proofs of Theorems \ref{thm:exact-gp-reduction} and \ref{thm:inducing-point-gp-restriction} now follow as simple corollaries of Theorems \ref{thm:exact-feature} and \ref{thm:inducing}. 

Recall the GP regression setting: $Y_i \mid f, x_i \sim P_f(y_i \mid x_i) = \mathcal{N}\left(f(x_i), \sigma^2\right)$ where $f\sim \mathcal{GP}(m,k)$. Assume that the GP prior admits a version a version with continuous sample paths, therefore taking $\mathcal{F} = C(\mathcal{X})$. Now, choose the feature map to be pointwise evaluation at the training inputs:
$
    \phi_X(f)
    \coloneqq
    \mathbf f_X
    =
    \bigl(f(x_1),\ldots,f(x_n)\bigr)^\top.
$
For each $\mbf{u}\in\mathbb R^n$ and $i=1,\ldots,n$, define the probability
kernel $
    \bar P_{\mbf{u}}^{x_i}(\mathrm dy)
    \coloneqq
    \mathcal N(\mathrm dy;\mbf{u}_i,\sigma^2)
$.
Then, for every $f\in\mathcal F$ and $i=1,\ldots,n$,
\[
\begin{aligned}
    \bar P_{\phi_X(f)}^{x_i}(\mathrm dy)
    &=
    \mathcal N\!\left(
        \mathrm dy;
        [\phi_X(f)]_i,
        \sigma^2
    \right) \\
    &=
    \mathcal N\!\left(
        \mathrm dy;
        f(x_i),
        \sigma^2
    \right)
    =
    P_f^{x_i}(\mathrm dy).
\end{aligned}
\]
Hence Assumption~\ref{as:feature-factorisation} holds with $d=n$ and $\phi=\phi_X$. Therefore, Theorem \ref{thm:exact-gp-reduction} is just a restatement of Theorem \ref{thm:exact-feature} in the GP regression setting with the log-score $S(P_Q^{x_i}, y_i) = - \log p_Q(y_i \mid x_i)$ and pushforward prior $\pi = \pi_X \coloneq (\phi_X)_\#\Pi$.

Finally, Theorem \ref{thm:inducing-point-gp-restriction} follows as a simple consequence of taking the inducing features to be evaluations of the latent function at a collection $\Lambda = (\lambda_1, \ldots, \lambda_m)$ of inducing points. These locations define the evaluation map $L_{\Lambda}:\mathcal F\to\mathbb R^m$ where
$L_{\Lambda}f
    \coloneqq
    \mathbf{f}_{\Lambda}$, with induced prior $\pi_{L_{\Lambda}} \coloneq \pi_{\Lambda} = \mathcal{N}\left(\mathbf{m}_{\Lambda}, K_{\Lambda \Lambda}\right)$. 

\section{ADDITIONAL DETAILS ON PRO-GPS}

\subsection{Variational characterisation of the Bayesian posterior} \label{ap:bayes-as-gibbs}

For $\lambda_n=n$, the variational objective in \eqref{eq:variational-bayes-posterior} is therefore
\begin{equation*}
    \mathcal L_n(Q)
    \coloneq
    -\sum_{i=1}^n
    \int_{\mathcal F}
    \log p_f(y_i\mid x_i)\,
    Q(\dd f)
    +
    \KL(Q\Vert\Pi).
\end{equation*}
Define the marginal likelihood
\[
    Z_n
    \coloneq
    \int_{\mathcal F}
    \prod_{i=1}^n
    p_f(y_i\mid x_i)\,
    \Pi(\dd f),
\]
and assume $0<Z_n<\infty$. The Bayesian posterior satisfies
\[
    Q_n^{\mathrm B}(\dd f)
    =
    \frac{1}{Z_n}
    \prod_{i=1}^n
    p_f(y_i\mid x_i)\,
    \Pi(\dd f),
\]
and hence, by the Radon–Nikodym theorem,
\[
    \frac{\dd Q_n^{\mathrm B}}{\dd\Pi}(f)
    =
    \frac{1}{Z_n}
    \prod_{i=1}^n
    p_f(y_i\mid x_i).
\]
For any $Q\ll\Pi$ for which the objective is finite,
\begin{align*}
    \KL(Q\Vert Q_n^{\mathrm B})
    &=
    \int_{\mathcal F}
    \log
    \left(
        \frac{\dd Q/\dd\Pi}
        {\dd Q_n^{\mathrm B}/\dd\Pi}
    \right)
    Q(\dd f) \\
    &=
    \KL(Q\Vert\Pi)
    -
    \sum_{i=1}^n
    \int_{\mathcal F}
    \log p_f(y_i\mid x_i)\,
    Q(\dd f)
    +
    \log Z_n \\
    &=
    \mathcal L_n(Q)+\log Z_n.
\end{align*}
Therefore,
\[
    \mathcal L_n(Q)
    =
    \KL(Q\Vert Q_n^{\mathrm B})
    -
    \log Z_n.
\]
Since $\log Z_n$ is independent of $Q$ and the KL divergence is uniquely minimised at $Q=Q_n^{\mathrm B}$,
\begin{equation*}
    Q_n^{\mathrm B}
    =
    \argmin_{Q\in\mathcal P(\mathcal F)}
    \mathcal L_n(Q).
\end{equation*}
Thus, when $\lambda_n=n$ the variational formulation recovers the Bayesian posterior exactly.

\subsection{Posterior Prediction for PrO-GPs}
\label{ap:gp-posterior-prediction}
Recall the form of the PrO posterior, given the optimised law $\tau_n$ in \eqref{eq:fs-pro-posterior}
\begin{equation*}
    Q_n^{\text{PrO}}(\mathrm df)
    =
    \int_{\mathbb{R}^n}
    \Pi(\mathrm df\mid \mbf{u})\,
    \tau_n(\mathrm d \mbf{u}),
\end{equation*}
By the conditional GP integration formulae in \eqref{eq:sn_squared}, we have $\Pi(\dd f \mid \mbf{u}) = \mathcal{GP}(\mu_{\mbf{u}}, s_X)$ as in Section \ref{sec:gps}. Then, by \eqref{eq:function-mixture-predictive}, the posterior predictive is
\begin{align*}
    P_{Q_n^\text{PrO}}^{x_\star}(\mathrm{d}y) &= \int_{\mathcal{F}} P_f^{x_\star}(\mathrm{d}y) Q_n^\text{PrO}(\dd f)  \\
     &= \int_{\mathcal{F}} P_f^{x_\star}(\mathrm{d}y ) \left[ \int_{\mathbb{R}^n}\Pi(\mathrm{df}\mid \mbf{u}) \tau_n(\mathrm{d} \mbf{u})  \right]\\ 
     &= \int_{\mathbb{R}^n} \left[\int_{\mathcal{F}} \mathcal{N}(\mathrm{d}y; f(x_\star), \sigma^2) \Pi(\mathrm{df}\mid \mbf{u}) \right]\tau_n(\mathrm{d} \mbf{u}) \\
     &=  \int_{\mathbb{R}^n} \left[\int_{\mathbb{R}} \mathcal{N}(\mathrm{d}y; a, \sigma^2) \mathcal{N}(\dd a; \mu_{\mbf{u}}(x_\star), s_X(x_\star,x_\star)) \right]\tau_n(\dd \mbf{u}) \\
     &= \int_{\mathbb{R}^n} \mathcal{N}(\mathrm{d}y; \mu_\mbf{u}(x_\star), s_X(x_\star,x_\star) + \sigma^2) \tau_n(\dd \mbf{u})
\end{align*}
In practice, $\tau_n$ is replaced by its particle approximation $\hat\tau_n^{1:K} = \frac{1}{K} \sum_{k=1}^K{\delta_{\mathbf{U}^{k}}}$ which yields the finite-particle posterior predictive used for evaluating NLPD in the experiments
\begin{equation}
    P_{Q_n^{\text{PrO}}}^{x_\star}(\dd y) \approx \frac{1}{K} \sum_{k=1}^K \mathcal{N}(\dd y ; \mu_{\mathbf{U}^{k}}(x_\star), s_X(x_\star, x_\star) + \sigma^2)
    \label{eq:pro-particle-predictive}
\end{equation}

\subsection{Connection to Parametric Gaussian Process Regressors} \label{ap:ppgr}

Parametric Gaussian Process Regressors introduce a ``predictively oriented'' objective $\mathcal{L}_{\text{ppgpr}}$ which bears close similarity to the objective in \eqref{eq:inducing-pro-gp-objective}. It is defined as \citep[Equation 18]{jankowiak2020parametric}
\begin{align}
    \mathcal{L}_{\text{ppgpr}} \coloneq\sum_{i=1}^n \log \varphi_{\sigma^2 + \sigma_q^2(x_i)}(y_i - \mu_q(x_i)) -\beta_{\text{reg} } \text{KL}(q \Vert \pi_\Lambda)
    \label{eq:ppgpr}
\end{align}
where $\varphi_{v}$ is a centred Gaussian with variance $v$ and
\begin{align*}
    q(\dd \mbf{u}) \coloneq \mathcal{N}(\dd\mbfu; \mbf{m}_q, \mbf{S}), \qquad
    \mu_q(x_i) \coloneq k_{x_i \Lambda}^\top K_{\Lambda \Lambda}^{-1} \mbf{m}_q, \qquad
    \sigma_q^2(x_i) \coloneq  r_{\Lambda}(x_i, x_i) +k_{x_i \Lambda}^\top K_{\Lambda \Lambda}^{-1}\mbf{S} K_{\Lambda \Lambda}^{-1} k_{\Lambda x_i},
\end{align*}
where we recall the definition of $r_{\Lambda}$ in \eqref{eq:r_lambda}. To keep our notation consistent throughout this work we have adapted the notation appearing in the original paper to our setting.

In fact, $\mathcal{L}_{\text{ppgpr}}$ is \textit{exactly} our objective \eqref{eq:inducing-pro-gp-objective}, with the restriction of the mixing measure $\tau$ to a Gaussian $q$. To show this, fix the inducing points $\Lambda$, kernel hyperparameters and $\sigma^2$. Recall the inducing-point PrO-GP objective:
\begin{align*}
    \mathcal{J}[\tau]
    &\coloneqq
    -\frac{\lambda_n}{n}
    \sum_{i=1}^n \log p_\tau(y_i \mid x_i)
    +
    \KL(\tau \| \pi_{\Lambda}), \\
    p_{\tau}({y_i \mid x_i})
&=
    \int_{\mathbb{R}^m}
    \varphi_{\sigma^2 + r_\Lambda(x_i,x_i)}
    \bigl(y_i-a_{\mbfu}(x_i)\bigr)
    \tau(\dd \mbfu).
\end{align*}
Now, restrict the inducing-feature law to the Gaussian family $
\tau(\dd\mbfu) = q(\dd\mbfu)$. In addition, assume a zero-mean GP so that $m(x_i) = 0$ and $\mbf{m}_{\Lambda} = \mbf{0}$. Then, $a_{\mbfu}(x_i)=k_{x_i \Lambda}^\top K_{\Lambda \Lambda}^{-1}\mbfu.$
Since the pushforward of $q$ through this linear map is Gaussian,
\begin{align*}
(a_{\cdot}(x_i))_\# q &= \mathcal{N} \left( k_{x_i \Lambda}^\top K_{\Lambda \Lambda}^{-1}\mbf{m}_q,\, k_{x_i \Lambda}^\top K_{\Lambda \Lambda}^{-1}\mbf{S} K_{\Lambda \Lambda}^{-1} k_{\Lambda x_i}\right). \\
&= \mathcal{N}(\mu_q(x_i), \sigma_q^2(x_i) - r_{\Lambda}(x_i, x_i))
\end{align*}
By pushforward change of variables,
\begin{align*}
&\int_{\mathbb{R}^m} \varphi_{\sigma^2 + r_{\Lambda}(x_i, x_i)} \bigl(y_i-a_{\mbfu}(x_i)\bigr)q(\dd\mbfu) \\
&\qquad=\int_{\mathbb{R}}\varphi_{\sigma^2 + r_{\Lambda}(x_i, x_i)}(y_i-a)\,\mathcal{N}
\left(\dd a; \mu_q(x_i),\, \sigma_q^2(x_i) - r_{\Lambda}(x_i, x_i)
\right)
\\
&\qquad=\varphi_{\sigma^2 +\sigma_q^2(x_i)}\left(y_i-\mu_q(x_i)\right),
\end{align*}
Hence, \eqref{eq:inducing-pro-gp-objective} becomes
\begin{align*}
    \mathcal{J}[q]
    &=
    -\frac{\lambda_n}{n}
     \sum_{i=1}^n \log \varphi_{\sigma^2 +\sigma_q^2(x_i)}\left(y_i-\mu_q(x_i)\right)
    +
    \KL(q \| \pi_{\Lambda}).
\end{align*}

Lastly, setting $
    \beta_{\text{reg}} = n /\lambda_n
$
and multiplying by this term gives the negative of $\mathcal{L}_{\text{ppgpr}}$. Hence, minimising the inducing-point PrO-GP objective with a Gaussian restriction to $\tau$ is equivalent to maximising the PPGPR objective.

\subsection{Replica--Gibbs algorithm and computational complexity}
\label{ap:computational-complexity}

\paragraph{Replica--Gibbs in whitened parameterisation} Assume we are given training inputs $X=(x_1,\ldots,x_n)$, inducing locations $\Lambda=(\lambda_1,\ldots,\lambda_m)$, and $K$ particles $\mathbf{U}^{1:K}=(\mbf{U}_1,\ldots,\mbf{U}_K)$. We treat the observation noise $\sigma$ and kernel hyperparameters $\gamma$ as fixed. Additionally, we use a zero-mean GP prior so that $\mbf{U}^k\sim\pi=\mathcal{N}(0,K_\pi).$

For numerical stability, we work in whitened coordinates. Let $K_\pi=L_\pi L_\pi^\top$ be the Cholesky factorisation of the prior covariance and write $\mbf{U}^k=L_\pi\mbf{Z}^k$ where $\mbf{Z}^k\sim\mathcal{N}(0,I_m).$
Thus, the algorithm operates directly on the whitened particles $\mbf{Z}^k$, with the corresponding inducing values recovered as $\mbf{U}^k=L_\pi\mbf{Z}^k$.
Let $\mbf{h}_i\in\mathbb{R}^m$ denote the linear coefficient appearing in the conditional predictive mean, so that
\[
    \mbf{h}_i^\top\mbf{U}^k = \mbf{h}_i^\top L_\pi\mbf{Z}^k = B_i^\top\mbf{Z}^k,
    \qquad
    B_i\coloneqq L_\pi^\top\mbf{h}_i.
\]
Writing $v_i$ for the conditional predictive variance and $C_{ik}$ for the allocation count associated with observation $i$ and particle $k$, the Gaussian update \eqref{eq:replica-gaussian-update} for whitened variables $\mathbf{z}$ becomes
\begin{align}
    \Sigma_k^{-1} &= I_m + \sum_{i=1}^{n}\frac{C_{ik}}{v_i}B_iB_i^\top, \notag\\
    \mbf{m}_k^\star &= \Sigma_k\sum_{i=1}^{n}\frac{C_{ik}y_i}{v_i}B_i.
    \label{eq:replica-gibbs-whitened}
\end{align}
The quantities $B_i$ can be precomputed. Let $H=(\mbf{h}_1,\ldots,\mbf{h}_n)\in\mathbb{R}^{m\times n}$ and define $B\coloneqq L_\pi^\top H=(B_1,\ldots,B_n).$
In the exact case, $m=n$, $H=I_n$, and therefore $B=L_X^\top,$ where $K_{XX}=L_XL_X^\top$. In the inducing case, $
    L_\pi=L_\Lambda,
    H=K_{\Lambda\Lambda}^{-1}K_{\Lambda X},
$
and hence
$
    B=L_\Lambda^\top K_{\Lambda\Lambda}^{-1}K_{\Lambda X}
    =L_\Lambda^{-1}K_{\Lambda X}.
$ See Algorithm \ref{alg:whitened-replica-gibbs} for the full whitened Replica--Gibbs sampling procedure.

\begin{algorithm}[t]
\caption{Whitened Replica--Gibbs sampler}
\label{alg:whitened-replica-gibbs}
\begin{algorithmic}
\State \textbf{Input:} training data $\{(x_i,y_i)\}_{i=1}^n$, inducing locations $\Lambda$, number of particles $K$, fixed kernel hyperparameters $\gamma$, fixed observation noise $\sigma$, score-weighting $\lambda_n$ and number of Gibbs iterations $T$
\State \textbf{Initialise:}
\State Compute the prior covariance $K_\pi$ and its Cholesky factorisation $K_\pi=L_\pi L_\pi^\top$.
\State Precompute $B=L_\pi^\top H$.
\State Sample the initial whitened particles $\mbf{Z}^k_{(0)}\sim\mathcal{N}(0,I_{m})$ independently for $k=1,\ldots, K$.
\State \textbf{For} $t=1,\ldots,T$
\State \hspace{1em}$\bullet$ Compute the particle predictions $B_i^\top \mbf{Z}^k_{(t-1)}$ for all $i=1,\ldots,n$ and $k=1,\ldots,K$.
\State \hspace{1em}$\bullet$ Sample the allocation counts $C_{ik}$ according to \eqref{eq:replica-allocations}.
\State \hspace{1em}$\bullet$ For each particle $k=1,\ldots,K$, compute the $\mbf{m}_k^\star$ and $\Sigma_k$ according to \eqref{eq:replica-gibbs-whitened}.
\State \hspace{1em}$\bullet$ Sample $\mbf{Z}^k_{(t)}\sim\mathcal{N}(\mbf{m}_k^\star,\Sigma_k)$ for $k=1,\ldots,K$.
\State \textbf{Return:} After $\widetilde{T}$ burn-in iterations, return unwhitened particles $\mbf{U}_{(t)}^k=L_\pi\mbf{Z}^k_{(t)}$ for $k=1,\ldots,K$ and $t > \widetilde{T}$.
\end{algorithmic}
\end{algorithm}

\paragraph{Computational complexity}

Assume we operate in the sparse regime $m \ll n$. The exact case can be recovered setting by $m = n$. The time complexity per sweep is dominated by forming the $K$, $m\times m$ precision matrices $\Sigma_k^{-1}$ (which involves $n$ summations) and subsequently factoring each of these matrices. Hence the dominant term is $\mathcal{O}(Knm^2 + Km^3)$. This compares favourably to SVGP, where the dominant cost for optimising hyperparameters is $\mathcal{O}(n_bm^2 + m^3)$ per iteration, where $n_b$ is the batch size. Although the additional $K$ term may appear prohibitive, the Replica--Gibbs updates are \textit{exact}, meaning we typically require as little as 50 iterations for convergence, compared to SVGP which may require many thousands. In addition, setting the batch size $n_b < n$ means that more iterations are required for a full pass over the data, which may therefore require more iterations for convergence. See Appendix \ref{ap:numerical-runtime-analysis} for numerical runtime analysis. Finally, the space complexity is dominated by the matrix $B \in \mathbb{R}^{m \times n}$, and hence is $\mathcal{O}(nm)$.

\paragraph{Choosing the number of particles}

Choosing the number of particles is a trade-off between flexibly representing the law $\tau_n$ and computational complexity.
Larger $K$ permits richer empirical approximations to the mixing law, but the computational cost enters multiplicatively into the dominant $\mathcal{O}(Knm^2)$ term and can therefore be prohibitive for large $n$ and $m$.
In practice, we find that $K=50$ provides a suitable tradeoff between robust predictive performance and time complexity throughout our experiments.

\subsection{Degenerate \texorpdfstring{$\sigma \to 0$}{sigma to 0} solutions during hyperparameter optimisation}
\label{ap:degenerate-sigma}
Assume we use the same dataset $\mathcal{D}_n$ both to fit the empirical particle law $\widehat{\tau}_{\mathbf{U}^{1:K}} \coloneq \sum_{i=1}^K \delta_{\mathbf{U}^k}$ and to optimise $\sigma$.
For fixed particles, define the empirical predictive score
\begin{equation*}
    \mathcal{R}_n(\widehat{\tau}_{\mathbf{U}^{1:K}}; \sigma)
    \coloneq
    -\frac{\lambda_n}{n}
    \sum_{i=1}^n
    \log\left\{
        \frac{1}{K}
        \sum_{k=1}^K
        \varphi_{v_i}(\epsilon_{ik})
    \right\},
\end{equation*}
where $\epsilon_{ik}\coloneq y_i-b_i-\mathbf h_i^\top\mathbf U^k$ and $v_i=r_i+\sigma^2$, with $r_i=0$ in the exact construction and $r_i=r_\Lambda(x_i,x_i)$ in the inducing construction.
Differentiating with respect to $\sigma$ gives
\begin{align*}
    \nabla_\sigma \mathcal{R}_n(\widehat{\tau}_{\mathbf{U}^{1:K}}; \sigma)
    =
    \frac{\lambda_n}{n}
    \sum_{i=1}^n
    \sigma
    \left[
        \frac{1}{v_i}
        -
        \sum_{k=1}^K
        \frac{\rho_{ik}\epsilon_{ik}^2}{v_i^2}
    \right],
\end{align*}
where $\rho_{ik}$ is defined in \eqref{eq:replica-allocations}.
For a given observation $i$, the corresponding contribution to the derivative is positive whenever $\sum_{k=1}^K\rho_{ik}\epsilon_{ik}^2<v_i$.
This can occur if a small set of particles predicts the observation particularly well: the responsibilities $\rho_{ik}$ are concentrated in a few particles with small residuals $\epsilon_{ik}^2$. Crucially therefore, for a sufficiently large number of particles $K$, there is enough flexibility such that the weighted residuals can be small for \textit{all} $y_i$ and the overall objective is positive. Gradient descent on this objective will therefore push $\sigma$ towards zero. In turn, smaller $\sigma$ will penalise the score term more heavily relative to the prior, meaning particles are encouraged even more to fit the training observations so that the residuals are small. This cycle drives $\sigma \to 0$ and all of the predictive variation is explained through the mixture over functions rather than observation noise.

Instead, splitting the data into two disjoint subsets $\mathcal{D}_{\text{train}}$ and $\mathcal{D}_{\text{val}}$, prevents overfitting and means the same observations can no longer be used both for fitting the particles and $\sigma$. Other regularisation strategies, such as imposing a prior on $\sigma$, restricting the number of particles or restricting the number of inducing points are possible. However, the primary drawback of all these approaches is the requirement for dataset-dependent tuning.
\subsection{Hyperparameter optimisation}
\label{ap:kernel-hyperparameters}

Let $k_\gamma$ denote a kernel family parameterised by $\gamma$. We wish to optimise the kernel hyperparameters $\gamma$ and observation noise $\sigma$ to improve predictive generalisation performance.
We initialise $\gamma\leftarrow\gamma_{\mathrm{init}}$ and $\sigma\leftarrow\sigma_{\mathrm{init}}$. These initial values may be chosen using standard GP defaults or obtained by maximising an objective for a standard GP, such as the marginal log-likelihood (MLL), leave-one-out cross-validation (LOO-CV), or the PPGPR objective. In particular, LOO-CV directly targets predictive performance for a single-particle approximation, while PPGPR corresponds to restricting the inducing mixing law $\tau$ to be Gaussian, as discussed in Appendix~\ref{ap:ppgr}. Therefore, these initialisations provide natural ``warm-start'' approximations to the PrO-GP hyperparameters.
As discussed in Appendix \ref{ap:degenerate-sigma}, we partition the data into $\mathcal D_{\mathrm{train}}$ and $\mathcal D_{\mathrm{val}}$ to avoid the particle law absorbing observation noise and driving $\sigma$ to zero. Denote by $\mathcal{R}_{\mathrm{train}}\left(\widehat{\tau}_{\mathbf U^{1:K}};\gamma,\sigma\right)$ the empirical predictive score evaluated on $\mathcal{D}_{\mathrm{train}}$ using parameters $\gamma$ and $\sigma$ and particle law $\widehat{\tau}_{\mathbf U^{1:K}}$. Denote by $\mathcal{R}_{\mathrm{val}}\left(\widehat{\tau}_{\mathbf U^{1:K}};\gamma,\sigma\right)$ the corresponding objective evaluated on $\mathcal{D}_{\mathrm{val}}$. The empirical particle law $\widehat{\tau}_{\mathbf U^{1:K}}$ and kernel hyperparameters $\gamma$ are fit using $\mathcal{R}_{\mathrm{train}}$, while $\sigma$ is updated using $\mathcal{R}_{\mathrm{val}}$.
We alternate between Replica--Gibbs updates of the particle law, gradient updates of $\gamma$ on the training objective, and gradient updates of $\sigma$ on the validation score. Since changing $\gamma$ changes the prior covariance, whitening transformation, conditional means, and conditional variances, the corresponding quantities are recomputed following each block of kernel-hyperparameter updates. Algorithm~\ref{alg:pro-gp-hyperparams} summarises the procedure.
Algorithm~\ref{alg:pro-gp-hyperparams} incurs an additional cost of repeated Cholesky factorisations of the updated kernel matrices, which is $\mathcal{O}(n m^2)$ per iteration.
\begin{algorithm}[t]
\caption{PrO-GP hyperparameter optimisation}
\label{alg:pro-gp-hyperparams}
\begin{algorithmic}
\State \textbf{Input:} data $\mathcal D_n=\{(x_i,y_i)\}_{i=1}^n$, inducing locations $\Lambda$, number of particles $K$, kernel family $k_\gamma$, initial hyperparameters $\gamma_{\mathrm{init}}$ and $\sigma_{\mathrm{init}}$, learning rates $\eta_\gamma$ and $\eta_\sigma$, number of adaptation sweeps $T_{\text{adapt}}$ and final Gibbs sweeps $T_{\mathrm{final}}$.
\State \textbf{Initialisation:}
\State Split $\mathcal D_n$ into training and validation sets $\mathcal D_{\mathrm{train}}$ and $\mathcal D_{\mathrm{val}}$.
\State Initialise $\gamma\leftarrow\gamma_{\mathrm{init}}$ and $\sigma\leftarrow\sigma_{\mathrm{init}}$.
\State Compute $K_\pi$ and its Cholesky factorisation $K_\pi=L_\pi L_\pi^{\top}$.
\State Compute the $\gamma$-dependent conditional quantities $H$, $B$, and predictive variances $\{v_i\}_{i=1}^n$.
\State Sample initial whitened particles $\mbf Z^k_{(0)}\sim\mathcal N(0,I_m)$ independently for $k=1,\ldots,K$.
\State \textbf{For} $t=1,\ldots,T_{\mathrm{adapt}}$:
\State \hspace{1em}$\bullet$ \textbf{Particle update:} perform a Replica--Gibbs sweep using $\mathcal D_{\mathrm{train}}$ and the current $(\gamma,\sigma)$.
\State \hspace{1em}$\bullet$ \textbf{Kernel update:} $
    \gamma \leftarrow \gamma-\eta_\gamma\nabla_\gamma
    \mathcal R_{\mathrm{train}}
    \left(\widehat{\tau}_{\mathbf U^{1:K}_{(t)}};\gamma,\sigma\right),
$
\State \hspace{1em}$\bullet$ Recompute $K_\pi$, $L_\pi$, $H$, $B$, and the predictive variances using the updated $\gamma$.
\State \hspace{1em}$\bullet$ \textbf{Noise update:} $
    \sigma \leftarrow \sigma-\eta_\sigma\nabla_\sigma
    \mathcal R_{\mathrm{val}}
    \left(\widehat{\tau}_{\mathbf U^{1:K}_{(t)}};\gamma,\sigma\right),
$
\State \textbf{Final particle update:} with $(\gamma,\sigma)$ fixed at their final values, perform $T_{\mathrm{final}}$ Replica--Gibbs sweeps.
\State \textbf{Return:} final hyperparameters $(\gamma,\sigma)$ and unwhitened particles $\mbf U^k_{(t)}=L_\pi^\gamma\mbf Z_{(t)}^k$, for $k=1,\ldots,K$ and $t > T_{\text{adapt}}$.
\end{algorithmic}
\end{algorithm}

\subsection{Discussion on \texorpdfstring{$\lambda_n$}{lambda	extunderscore n}}
\label{ap:lambda_n_discussion}

The $\lambda_n$ parameter describes the trade-off between the prior KL regularisation term and the data fitting term. In practical applications, it may be preferable to tune $\lambda_n$ to achieve the best possible performance. In the parametric setting, standard assumptions support scaling $\lambda_n = c\cdot\sqrt{n}$ for some $c \in \mathbb{R}^+$, up to logarithmic factors \citep{mclatchie2025predictively}. In particular, it is necessary that the effective complexity of the model is fixed, whilst $n$ increases. However, for PrO-GPs the dimensionality of the model and the effective complexity of the prior increase with either $n$ or the number of inducing points $m$. Nevertheless, for any \textit{fixed} regression task at hand and any particular GP prior we note:
\vspace{-0.5em}
\begin{itemize}
    \item For a fixed number of inducing points, the effective complexity of the model is fixed. Hence, cross-validation over $c$ in some range of values specific to the problem at hand is well-justified. 
    \item For problems wherein $m \ll n$, the score term dominates the KL prior. Apart from atypical problems where especially heavy regularisation towards the prior is beneficial, the impact of the $\lambda_n$ parameter is negligible.
\end{itemize}
In summary, we anticipate this is useful guidance for practitioners, but leave detailed theoretical treatment for future work.

\subsection{Sampling via the Wasserstein Gradient Flow}
\label{ap:wgf-sampling}

The Replica-Gibbs sampler presented in Section \ref{sec:computation} exploits the log-score structure for the replica augmentation, and the Gaussian observation model for conjugate particle updates. Although extensions to handle non-Gaussian likelihoods would be possible, the log-score is indispensable. Therefore, we consider the more general objective \eqref{eq:full-function-pro} and present an alternative sampling scheme based on the Wasserstein Gradient Flow (WGF) \citep{ambrosio2005gradient} for PrO-GP computation with proper scoring rules $S: \mathcal{P}(\mathcal{Y}) \times \mathcal{Y} \to \mathbb{R} \cup \{\infty \}$.  

Suppose that each score contribution $S_i(\tau) \coloneq S(P_\tau^{x_i}, y_i)$ admits a first variation $\frac{\delta S_i}{\delta\tau}(\tau)(\mbf{u})$ which is differentiable in $\mbf{u}$. 
Define $\ell_{\tau}(\mbf{u}) \coloneq \frac{\lambda_n}{n} \sum_{i=1}^n \frac{\delta S_i}{\delta\tau}(\tau)(\mbf{u})$. 
Recall the definition of $\mathcal{J}_S^{\pi}$ in \eqref{eq:full-function-pro}. Its Wasserstein gradient is given by
\begin{align*}
    \quad \nabla_{W_2}\mathcal{J}_{S}^{\pi}[\tau_t] &= \nabla_{W_2} \frac{\lambda_n}{n}\sum_{i=1}^nS_i(\tau_t) + \nabla_{W_2} \text{KL}(\tau_t \Vert \pi) =\nabla_\mbf{u}\ell_{\tau_t} + \nabla_\mbf{u}\log \frac{\tau_t}{\pi}.
\end{align*}
The WGF is then defined as the weak solution $\tau_t$ of the continuity equation
\begin{equation*}
    \partial_t \tau_t= \nabla \cdot \left( \tau_t \nabla_{W_2}\mathcal{J}_S   ^{\pi}[\tau_t] \right).
    \label{eq:continuity-wgf}
\end{equation*}
To implement this numerically, expand the above to give
\begin{equation*}
    \partial_t \tau_t = \nabla_\mbf{u} \cdot [\tau_t(\nabla_\mbf{u} \ell_{\tau_t} - \nabla_\mbf{u} \log \pi)] + \Delta\tau_t.
\end{equation*}
This is the nonlinear Fokker-Planck equation associated with the McKean-Vlasov SDE:
\begin{align*}
    \mathrm{d}\mbf{U}_t &= -[\nabla_\mbf{u} \ell_{\tau_t}(\mbf{U}_t) - \nabla_\mbf{u} \log \pi(\mbf{U}_t)] \mathrm{d}t + \sqrt{2} \mathrm{d}B_t,
\end{align*}
where $\tau_t = \text{Law}(\mbf{U}_t)$ and $B_t$ is Brownian motion. We approximate this SDE by introducing $K$ particles $\mbf{U}^{1:K} = (\mbf{U}^1, \ldots, \mbf{U}^K) \in (\mathbb{R}^n)^K$ and substituting the empirical measure $\hat{\tau}_t^{1:K} \coloneq \frac{1}{K} \sum_{k=1}^K \delta_{\mbf{U}_t^k}$. This gives the interacting system
\begin{align*}
    \mathrm{d}\mbf{U}^k_t = - \left[ \nabla_\mbf{u} \ell_{\hat{\tau}^{{1:K}}_t}(\mbf{U}_t^k) - \nabla_\mbf{u} \log \pi(\mbf{U}_t^k) \right] \mathrm{d}t + \sqrt{2}\mathrm{d}B_t^k
\end{align*}
where each $B^{k}$ is an independent Brownian motion on $\mathbb{R}^n$. Finally, these dynamics can be simulated numerically via an Euler--Maruyama discretisation with step size $\eta$. 
Under mild regularity conditions, this algorithm converges to the minimiser of the PrO objective, with additional error arising from finite-particle approximation and time discretisation \citep{suzuki2023convergence}.
Specialising to PrO-GPs with the exact reduction and Gaussian likelihood, we have $S_i(\tau) = - \log \int \varphi_{\sigma^2}(y_i - \mbf{v}_i) \tau(\dd \mbf{v})$, and so
\begin{align*}
    \nabla_{\mbf{u}} \ell_{\tau}(\mbf{u}) = - \frac{\lambda_n}{n}\sum_{i=1}^n \frac{\varphi_{\sigma^2}(y_i - \mbf{u}_i)}{\int \varphi_{\sigma^2}(y_i - \mbf{v}_i) \tau(\mathrm{d}\mbf{v})} \frac{y_i-\mbf{u}_i}{\sigma^2} \mbf{e}_i
\end{align*}
Assuming the prior is $\mathcal{N}(0, K_{XX})$,
then
$
    \nabla_{\mbf{u}} \log \pi(\mbf{u}) = -K_{XX}^{-1}\mbf{u}
$
and the interacting particle system becomes
\begin{align*}
    \mathrm{d}\mbf{U}^k_t
    =
    \left[
        \frac{K\lambda_n}{n}
        \sum_{i=1}^n
        \frac{
            \varphi_{\sigma^2}(y_i-\mbf{U}_{t,i}^k)
        }{
            \sum_{j=1}^K
            \varphi_{\sigma^2}(y_i-\mbf{U}_{t,i}^j)
        }
        \frac{y_i-\mbf{U}_{t,i}^k}{\sigma^2}
        \mbf{e}_i
        -
        K_{XX}^{-1}\mbf{U}_t^k
    \right]
    \mathrm{d}t
    +
    \sqrt{2}\mathrm{d}B_t^k.
\end{align*}
A whitened representation can also be used to reduce the instability from repeated applications of $K_{XX}^{-1}$. Compared to Algorithm \ref{alg:whitened-replica-gibbs}, this scheme introduces additional step-size parameter $\eta$, which may be difficult to tune in practice \citep[see, e.g.,][]{sharrock2025tuning}. It also takes gradients of the log-score, resulting in a term proportional to $\frac{1}{\sigma^2}$. As a result, we find this approach introduces additional error, is slower to converge, and is significantly less stable. 

\section{ADDITIONAL EXPERIMENTAL DETAILS}
\label{ap:experimental-details}

This section describes the experimental setup used in Section \ref{sec:Experiments}. All experiments were run on a single desktop machine with an AMD Ryzen 5 7600 CPU and 24 GB of RAM, using 8 CPU threads. Our implementation uses \href{https://github.com/QuantClimate/GPJax}{\texttt{GPJax}} \citep{Pinder2022} for the GP models and \href{https://github.com/blackjax-devs/blackjax}{\texttt{BlackJAX}} \citep{cabezas2024blackjax} for sampling.
%

\subsection{Negative Log Predictive Density}
\label{ap:nlpd}

 Let $(x_\star, y_\star)$ be a test point, such that $y_\star \sim P_0^{x_{\star}}$. For \GP{}s, $\text{NLPD}_{\GP}(x_\star, y_\star) \coloneq  - \log \varphi_{v_\star}(y_\star - \mu_\star)$, where $\mu_\star$ and $v_\star$ are the predictive mean and variance specified by the model. For \PROGP{}s, $\text{NLPD}_{\PROGP{}}(x_\star, y_\star) \coloneq -\log \frac{1}{K} \sum_k \varphi_{s_X(x_\star, x_\star) + \hat\sigma^2}(y_* - \mu_{\mathbf{U}
 _{k}}(x_\star))$, where $\mu$ and $s_X$ are defined in \eqref{eq:sn_squared}. 
This is just the negative log of the particle approximation to the PrO predictive in $\eqref{eq:pro-predictive}$.

\subsection{Synthetic experiments from Section \ref{sec:synthetic}}
\label{ap:synthetic}

We measure the held-out NLPD of \PROGP{} vs \GP{}, as a function of training set size $n$, with $n \in [100, 200, 400]$. We use a synthetic 1-D generator in four different regimes: `Block outliers', 'Heteroskedastic', 'Multimodal' and 'Well-specified'. For all regimes, we use an RBF kernel $k_{\ell, \alpha} = \alpha^2 \exp{\left( -\frac{||x_i-x_j||^2}{2\ell^2} \right)}$,  draw a lengthscale $\ell \sim \mathcal{U}[0.5, 1.0]$, standard deviation $\alpha \sim \mathcal{U}[0.5, 2.0]$ and input $X \sim \mathcal{U}[-2,2]^n$. We then draw a latent function $f \sim \mathcal{GP}(0, K_{XX})$, $K_{XX} = k_{\ell, \alpha}(x_i, x_j)$, and generate $y= f(X) + \epsilon, \epsilon \sim\mathcal{N}(0, \sigma^2)$ with $\sigma = 0.15 \alpha$. We then generate the misspecified data in each regime as follows:
\begin{itemize}
    \item \textbf{Block outliers} Split points 70/30 into train/test uniformly at random. Sample one contiguous outlier region $\mathcal{R} = [c-w, c+w] \subset [x_{\min}, x_{\max}]$, with center $c \sim \mathcal{U}[x_{\min}, x_{\max}]$ and width $0.4$. Corrupt only the training points that fall in $\mathcal{R}$: draw a sign $s \in [-1, 1]$, and add an offset $s \cdot \, \alpha$ to $y$ at those points.
    \item \textbf{Multimodal} Draw an additional two independent functions
    $f_a, f_b \sim \mathcal{GP}(0, K_{XX})$ and a branch center $c \sim \mathcal{U}[x_{\min}, x_{\max}]$.
    Blend each into $f$ with a weight $w(x)$ that is $0$ near $c$ and grows to $1$ away from it,
    giving two branches $f_a' = f + w(x) f_a$, $f_b' = f + w(x) f_b$ that agree near $c$ and diverge
    elsewhere. Assign each point to branch $a$ or $b$ independently with probability $0.5$,
    $z \sim \text{Bernoulli}(0.5)$, giving $y_{\text{truth}} = z f_b' + (1-z) f_a'$. Split 70/30
    into train/test uniformly at random.
    \item \textbf{Heteroskedastic} Sample a region with center $c \sim \mathcal{U}[x_{\min}, x_{\max}]$
    and inflate the noise locally: $\sigma(x)^2 = \sigma^2 + 0.4\,\alpha \cdot
    b(x)^2$, where $b$ is a bump equal to $1$ at $c$ and decaying to $0$ over its width. Draw $y = f(X) + \epsilon(x)$,
    $\epsilon(x) \sim \mathcal{N}(0, \sigma(x)^2)$, replacing the constant-$\sigma$ noise model above
    with this position-dependent one. Split 70/30 into train/test uniformly at random.
\end{itemize}
For each regime, fit \GP{} and \PROGP{} models and compute the NLPD on the held-out test data. We repeat this over 20 splits to report the mean NLPD and 95\% credible interval. For \GP{}, we optimise the kernel-hyperparameters and observation noise using the L-BFGS-B optimiser. For \PROGP{}, we fix the kernel-hyperparameters to the \GP{} values and use 100 adaptation steps for tuning the observation noise $\sigma$, followed by 30 final sampling steps. In full, we use Algorithm \ref{alg:pro-gp-hyperparams} with the following configuration:

\begin{itemize}
    \item 70/30 random split for $\mathcal{D}_{\text{train}}/\mathcal{D}_{\text{val}}$
    \item $K = 50$ and $\lambda_n = n$
    \item $\sigma_{\mathrm{init}}$ and $\gamma_{\mathrm{init}}$  to the \GP{} optimised values.
    \item $T_{\text{adapt}} = 100$, $T_{\mathrm{gibbs}} = 30$, $\eta_{\sigma} = 0.1$ and $\eta_{\gamma} = 0$ (i.e. take \textit{no} kernel adaptation steps)
\end{itemize}

\subsection{Traffic-flow experiment from Section \ref{sec:traffic-flow}}
\label{ap:real-examples-1}

We use the data and plotting tools available in the \href{https://github.com/vabor112/pems-regression/tree/main/pems_regression}{\texttt{pems-regression}} library which contains a graph of the San Jose freeway sensor network with 1,016 nodes and 1,173 edges. Traffic speed measurements are available at 325 sensor locations. We use $n=200,225,250$ and $275$ sensors for training as shown in Table \ref{tab:graph-gp}, and report average held-out NLPD on 5-fold random splits.

We train a \GP{} with a graph Mat\'{e}rn kernel \citep{borovitskiy2021matern}, using the implementation provided in \href{https://github.com/QuantClimate/GPJax}{\texttt{GPJax}} \citep{Pinder2022}. Following the approach in the \texttt{pems-regression} library,  we fit the graph \GP{} with five random restarts. For each restart, the smoothness and lengthscale are drawn from $\mathcal{U}[1,10]$ and the kernel variance is initialised to 1. These parameters and the observation variance are optimised by maximising the MLL with Adam, with learning rate 0.1 and 200 iterations. We keep the restart with the highest training MLL. We then fit a \PROGP{} with the kernel-hyperparameters fixed to those of the selected \GP{}, following the approach and parameters as described for the synthetic regression experiment in Appendix \ref{ap:synthetic}. The results of the posterior predictive inferences, overlaid onto the road network, are displayed in Figure \ref{fig:road_grid}\footnote{Unlike the corresponding figure in \citet{borovitskiy2021matern}, which shows only the posterior standard deviation of the latent function, our figure shows the standard deviation of the whole posterior predictive distribution. This includes both the uncertainty about the function and the observation noise.}. Circles represent sensor nodes and dots represent training nodes (i.e. sensors where speed measurements are available).

\begin{figure}[h!]
\centering
\includegraphics[width=0.85\textwidth]{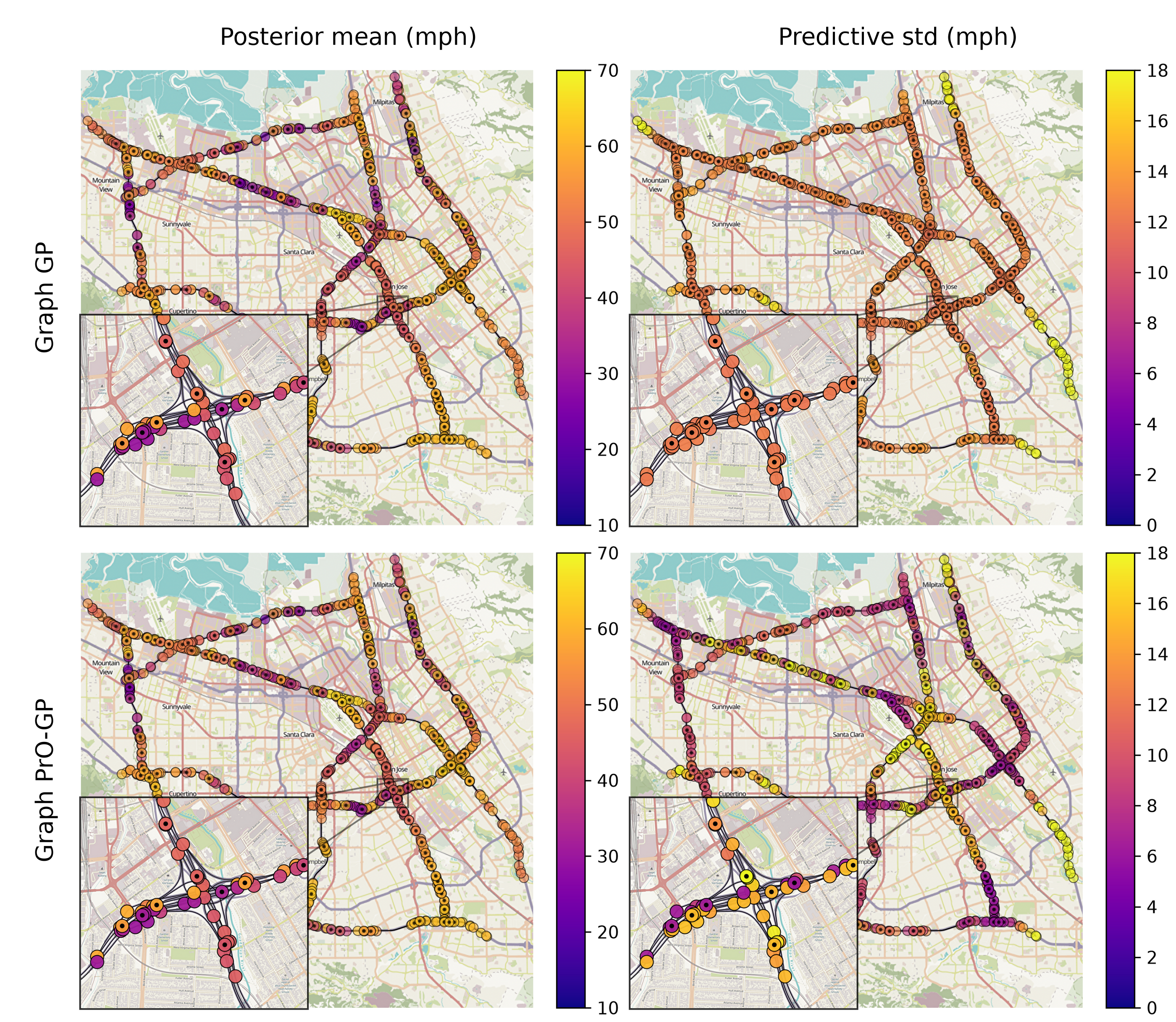}
    \caption{Graph \GP{} and Graph \PROGP{} posterior inference on the San Jose road network. Zoomed in predictions are in the bottom left of each plot. Circles represent sensor nodes. Black dots represent training nodes. \emph{Top left:} Graph \GP{} predictive mean. \emph{Top right:} Graph \GP{} predictive standard deviation. \emph{Bottom left:} Graph \PROGP{} predictive mean. \emph{Bottom right:} Graph \PROGP{} predictive standard deviation}
    \label{fig:road_grid}
\end{figure}

\subsection{UCI regression experiments from Section \ref{sec:uci-regression}}
\label{ap:uci-regression}
\paragraph{Experimental setup} We partition the data into 80/20 train/test sets and generate a random 10-fold split for each dataset. All datasets are standardised in output space. Throughout, we use an RBF kernel with independent lengthscales per dimension and a prior mean of zero. Lengthscale values are initialised as the square root of the input data dimension and are bounded in the range $[0.001, 1000]$. For each dataset with $n \leq 1030$ we fit a \GP{} model using the marginal log-likelihood and L-BFGS-B optimiser. For datasets with $1030 < n < 45730$, we use $m=200$ inducing points initialised with kmeans, and fit the \SVGP{} and \PPGPR{} model variants using an Adam optimiser with an initial learning rate of 0.1, batch-size 512 and 10000 iterations. For \textit{kin40k} and \textit{protein} we use $m = 500$ inducing points. We set $\beta_{\text{reg}} =1$ in the \PPGPR{} objective, which corresponds to the same weighting $\lambda_n /n = 1$ used for \PIPROGP{}. Whilst \citet{jankowiak2020parametric} propose several model variants based on their objective in \eqref{eq:ppgpr}, we compare against ``PPGPR-Chol'', which employs a full-rank Gaussian covariance and is therefore the closest Gaussian counterpart to our method. This allows us to isolate the gains from extending the mixing law beyond the Gaussian family. We fit the \PROGP{} and \IPROGP{} model variants using the approach and parameters in Appendix \ref{ap:synthetic};  $\gamma_{\mathrm{init}}$ is initialised from the \GP{}/\SVGP{}/\PPGPR{} optimised values (for \PROGP{}/\SIPROGP{}/\PIPROGP{}, respectively), and all other parameters are initialised as in the synthetic experiment.

\paragraph{Machine and stock analysis} As shown in Table \ref{table:uci-results}, \PROGP{} reduces NLPD in all datasets except machine and stock. For both datasets (and \emph{concreteslump} for comparison), we display residual plots of the Cholesky-whitened \GP{} predictions in Figure \ref{fig:quantile-plots}. Under the GP prior, these should be normally distributed.

For \emph{machine}, the whitened residuals are broadly consistent with normality, with only a small departure in the left tail. For \emph{stock}, the centre is also mostly normal with departures in both tails, suggesting slightly heavier tails than a Gaussian. In contrast, \emph{concreteslump} displays significantly more right-skewness and much heavier tails. This aligns with the results from Table \ref{table:uci-results} and theory for finite-parameter PrO posteriors: when the model is already well-specified, fitting a PrO posterior offers no advantage in predictive performance.

\begin{figure}[h!]
\centering
\includegraphics[width=0.6\textwidth]{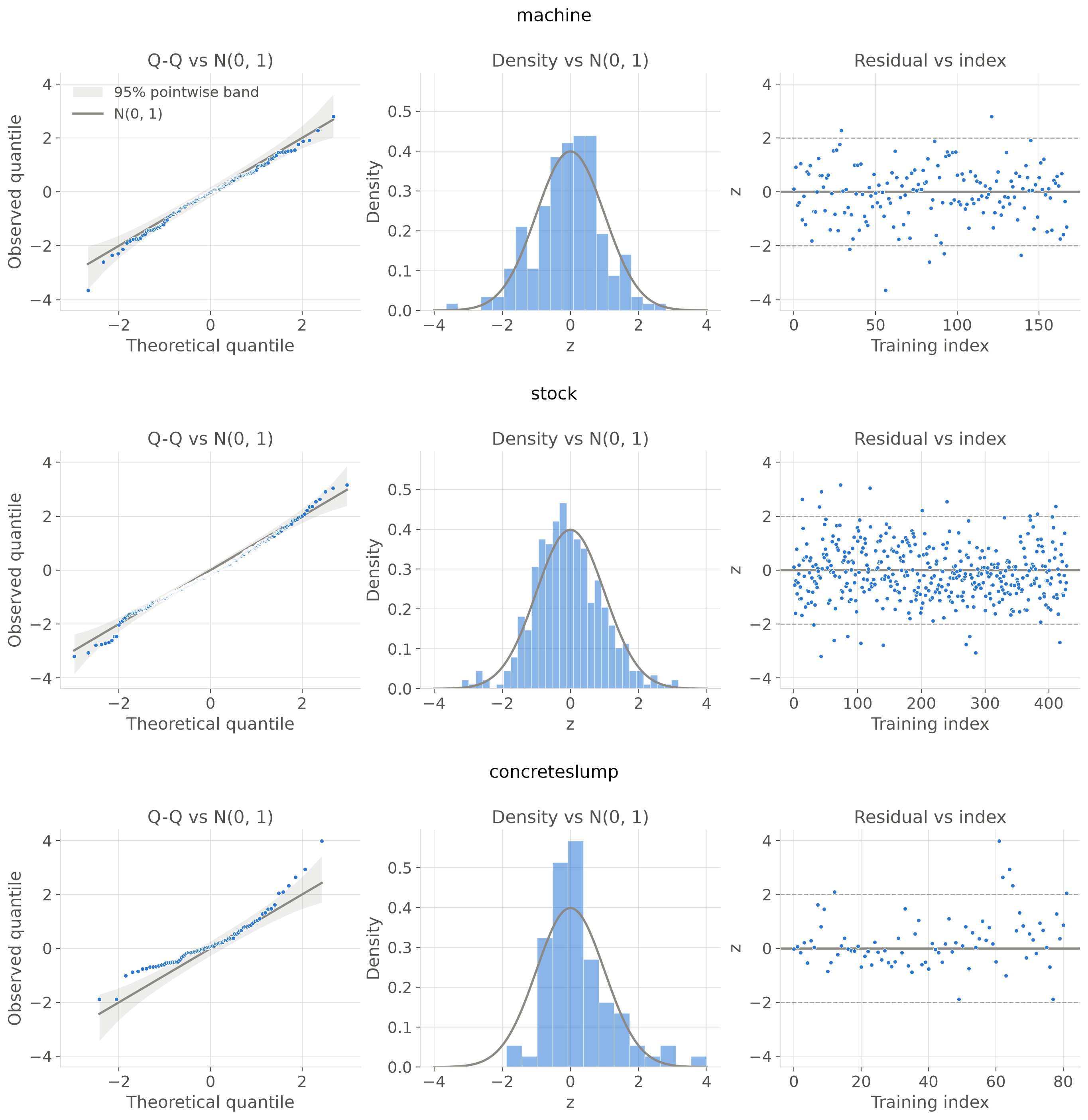}
    \caption{Diagnostic plots of whitened GP residuals for the machine, stock, and concreteslump datasets. Left: normal Q–Q plots with 95\% pointwise reference bands. Centre: residual histograms overlaid with the standard normal density. Right: residuals against training index, with reference lines at 0 and ±2 standard deviations}
    \label{fig:quantile-plots}
\end{figure}

\subsection{Connection to explicit mixture models}
\label{ap:mixture-comparison}

 Recall the form of the particle approximation to the PrO-GP predictive in \eqref{eq:pro-particle-predictive}:
\begin{align*}
    P_{Q_n^{\text{PrO}}}^{x_\star}(\dd y) \approx \frac{1}{K} \sum_{k=1}^K \mathcal{N}(\dd y ; \mu_{\mathbf{U}^{k}}(x_\star), s_X(x_\star, x_\star) + \sigma^2),
\end{align*}
and the Gibbs measure used to target this predictive in \eqref{eq:replica-gibbs-target}:
\begin{equation*}
    \tau_{n,K}(\dd\mbf{U}^{1:K})
    \propto
    \prod_{k=1}^K\pi(\dd\mbf{U}^k)
    \prod_{i=1}^n
    \Big\{
    \frac{1}{K}\sum_{k=1}^K g_i(\mbf{U}^k)
    \Big\}^{\frac{K\lambda_n}{n}}.
    \vspace{2mm}
\end{equation*}

In this section we compare PrO-GPs to the Overlapping Mixture of Gaussian Processes
(OMGP) \citep{lazaro2012overlapping}, the explicit mixture model whose predictive is
closest in form to \eqref{eq:pro-particle-predictive}. The OMGP posits $M$ independent latent
functions $f_1,\dots,f_M \sim \mathcal{GP}(m,k)$, mixing weights
$\boldsymbol{\omega} = (\omega_1,\dots,\omega_M)$, and an
assignment $z_i \in \{1,\dots,M\}$ for each observation, with
\[
  z_i \sim \mathrm{Cat}(\boldsymbol{\omega}), \qquad
  y_i \mid z_i = m,\, f_m \sim \mathcal{N}\big(f_m(x_i), \sigma^2\big).
\]
Writing each component in the same finite feature basis as Section~\ref{sec:computation},
so that $\mbf{U}^m \sim \pi$ independently, and marginalising the assignments, the
OMGP posterior for fixed $\boldsymbol{\omega}$ and $\sigma$ is
\begin{equation}
  Q^{\text{OMGP}}_n(\dd\mbf{U}^{1:M} \mid y)
  \;\propto\; \prod_{m=1}^M \pi(\dd \mbf{U}^m)
  \prod_{i=1}^n \Big\{ \sum_{m=1}^M \omega_m\, g_i(\mbf{U}^m) \Big\},
  \label{eq:omgp-post}
\end{equation}
with $g_i$ the same Gaussian density as in \eqref{eq:common-finite-dim-opt}. The
OMGP posterior predictive at $x_\star$ is
$\sum_{m} \omega_m\, \mathcal{N}\big(\mu_{\mbf{U}^m}(x_\star),\, s_X(x_\star,x_\star) + \sigma^2\big)$,
averaged over $Q_n^{\text{OMGP}}$. Hence, comparing \eqref{eq:omgp-post} with the particle target \eqref{eq:replica-gibbs-target}, the link is exact: with $M = K$ and uniform weights $\omega_m = 1/K$, $\tau_{n,K}$ is the OMGP posterior with its likelihood raised to the power $r = K\lambda_n/n$.

However, there are significant conceptual and practical differences, which we note here. In OMGPs the mixture is an explicit \textit{parameter} of the model and the posterior we target changes with $M$. In PrO-GPs, the mixture is the posterior itself and $K$ is purely a numerical setting: as $K \to \infty$ the particle approximation converges to the PrO-GP posterior. Furthermore, the tempering factor $r$ means that each particle is informed by roughly $\lambda_n$ effective points regardless of $K$. In contrast, for OMGPs each component is informed by roughly $n/M$ points; as $M$ grows, the components will become increasingly prior-dominated. Lastly, we emphasise that this connection materialises due to the specific use of the log-score in the PrO-GP objective. Under other proper scoring rules, such as those permitted by our general objective in Appendix~\ref{ap:theoretical-proofs}, the score of $P_Q$ is no longer the likelihood of any tempered Bayesian mixture model, and the correspondence with OMGP disappears.

\paragraph{Numerical evaluation.} Aside from these conceptual differences, we further perform an experiment to isolate the difference between changing the inference and changing the model. To do so, we use OMGP as a controlled comparison. We fix the kernel-hyperparameters and inducing basis from the corresponding \GP{} / \SVGP{} fit and enrich the model with $M$ components. Whilst \citet{lazaro2012overlapping} fit OMGP variationally, the form of the OMGP target in \eqref{eq:omgp-post} allows us to leverage the Replica--Gibbs sampler and remove confounding from the choice of sampling scheme. Rather than optimising $\gamma$ and $\sigma$ on the variational bound, we place a $\mathrm{Dir}(a\mathbf{1}_M)$ prior on $\boldsymbol{\omega}$ and an Inverse-Gamma prior on $\sigma^2$ centred on the \GP{} / \SVGP{} estimate and shared across components, and sample both within each Gibbs sweep. We evaluate OMGP with $M \in \{1,2,3,5,10,20\}$, $a = 1$, and $1000$ sweeps. Henceforth, we denote the OMGP model with this configuration by $\textsc{omgp}$. We compare against the \PROGP{} and \SIPROGP{} on the same datasets and splits in Section~\ref{sec:uci-regression} using held-out NLPD. The results are presented in Figure \ref{fig:omgp-inducing}.

Firstly, we observe that adding components improves the \textsc{omgp} over $M = 1$ on every dataset, confirming that a mixture predictive is beneficial on these tasks. However, the best value of $M$ is strongly dataset dependent, ranging from $M = 2$ (\textit{machine}, \textit{stock}) to $M =20$ (\textit{protein}), and performance is often not monotone in $M$: on \textit{autos}, \textit{machine}, \textit{stock}, \textit{energy}, \textit{autompg} and \textit{parkinsons}, the NLPD worsens for large $M$ as the superfluous components compete over a fixed number of data assignments. This is most visible on \textit{machine} and \textit{stock}, which we identified in Appendix~\ref{ap:uci-regression} as close to well specified: here the \PROGP{} remains close to the \GP{}, whereas the \textsc{omgp} degrades steadily as $M$ grows.

Comparing against the \textsc{omgp} at its best $M$, chosen favourably with hindsight on the test set,
\PROGP{}s are better on six datasets (\textit{concreteslump}, \textit{autos}, \textit{wine}, \textit{elevators}, \textit{kin40k} and \textit{protein}), at every $M$ considered, and comparable on five (\textit{machine}, \textit{energy}, \textit{concrete},
\textit{abalone} and \textit{airquality}). The \textsc{omgp} is better on the remaining five, but by a small margin ($\leq 0.02$) on \textit{autompg}, \textit{skillcraft} and \textit{stock}; only on \textit{housing} and \textit{parkinsons} does the best \textsc{omgp} improve by more than $0.1$. Overall,
an explicit mixture can match \PROGP{}s, but only once $M$ is tuned per
dataset, whereas \PROGP{}s achieve this with a fixed number of particles and no choice of
mixture size.

\begin{figure}[t]
\centering
\includegraphics[width=0.9\textwidth]{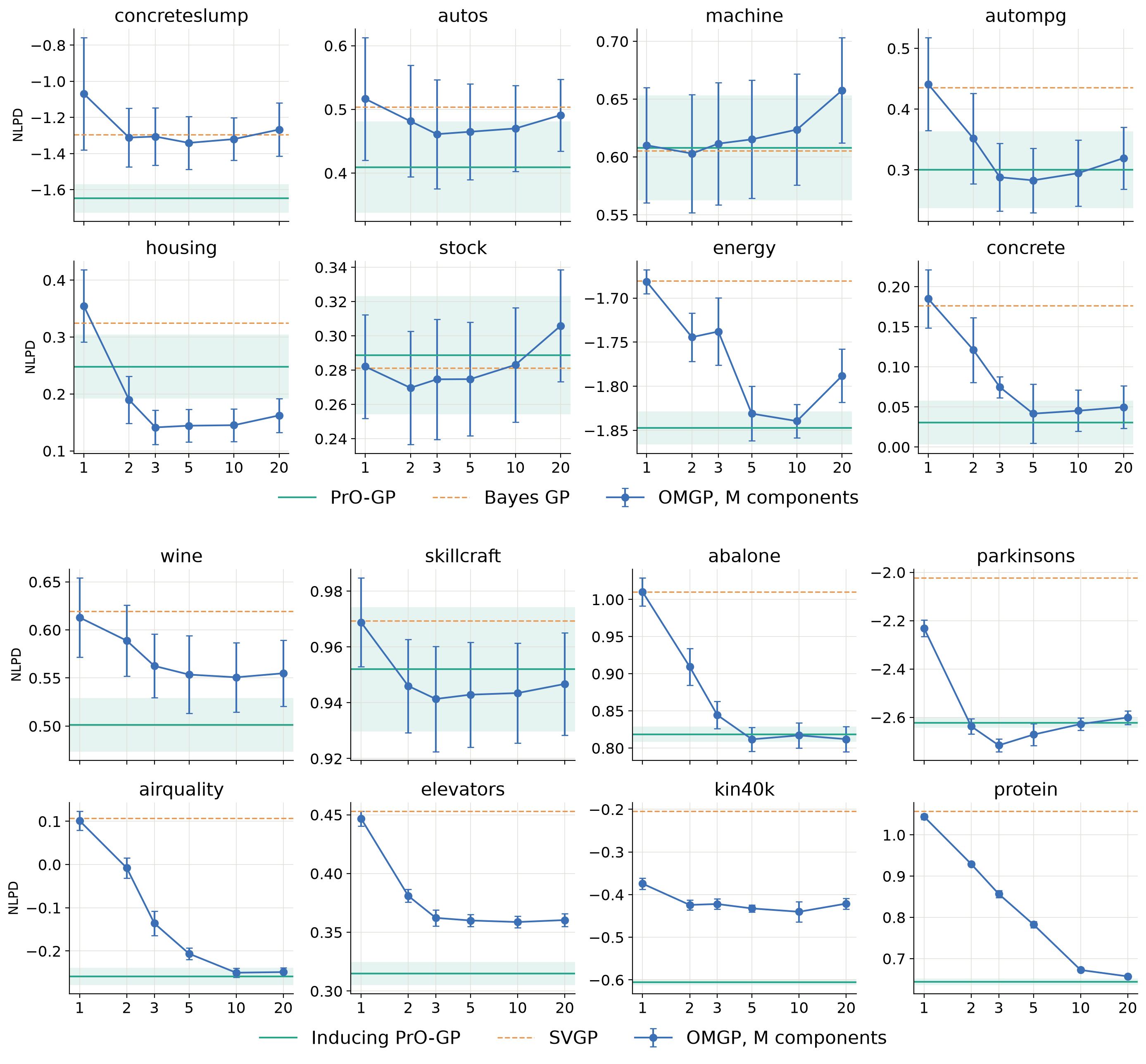}
    \caption{Held-out NLPD on 16 UCI regression datasets. Blue: \textsc{omgp}
as a function of the number of mixture components $M$, with kernel-hyperparameters and inducing points fixed to the \GP{} or \SVGP{} fit. Green line, solid: \PROGP{} or \SIPROGP{} with
$K = 50$ particles. Orange line, dashed: \GP{} or \SVGP{}. Error bars and shaded bands show 1 standard error.}
\label{fig:omgp-inducing}
\end{figure}

\subsection{Numerical runtime analysis}
\label{ap:numerical-runtime-analysis}

In this section we analyse the runtime of the \IPROGP{} variant numerically. Our primary goal is to determine the \textit{additional} cost of fitting a PrO posterior compared to training an \SVGP{}. Therefore, in line with the methodology presented throughout the paper, we consider \IPROGP{} hyperparameters fixed and do not consider the runtime of additional hyperparameter optimisation steps. We use \textit{parkinsons} as the benchmarking dataset and break the total runtime down into two components. First, we analyse the cost-per-iteration of Replica--Gibbs for \IPROGP{} compared to gradient ascent for \SVGP{}. Second, we analyse the number of iterations for convergence by computing the average value of the \IPROGP{} objective per Replica--Gibbs sweep. Both values are computed for $m=100, 250, 500, 1000$ inducing points. For \SVGP{} we use a batch size $b=512$ and for \IPROGP{} we use $K=50$ particles. The results are presented in Figure \ref{fig:runtime}.

The cost-per-iteration is roughly 10 times greater for an \IPROGP{} Replica--Gibbs sweep compared to an \SVGP{} gradient step, with a similar trend as the number of inducing points increases. This aligns with the expectation from theoretical time complexity (see Appendix \ref{ap:computational-complexity}); the additional cost-per-iteration is directly proportional to the number of particles $K$. Furthermore, since we used batch size $b = 512 < 5875 = n$, \SVGP{} requires $\approx 10$ iterations for a full pass of the dataset. Importantly, observe that \IPROGP{} training score plateaus after only $\approx 50$ Replica--Gibbs sweeps, with  diminishing returns thereafter. Hence, we expect the number of Replica--Gibbs sweeps for convergence is approximately $\mathcal{O}(10)$. In comparison, \SVGP{} typically requires $\mathcal{O}(10^3)$ to $\mathcal{O}(10^4)$ gradient ascent steps for convergence. Since the total number of sweeps required is considerably less, the additional time required to fit an \IPROGP{} is small despite the increase in cost-per-iteration.
\vspace{2em}
\begin{figure}[h!]
\centering
\includegraphics[width=0.75\textwidth]{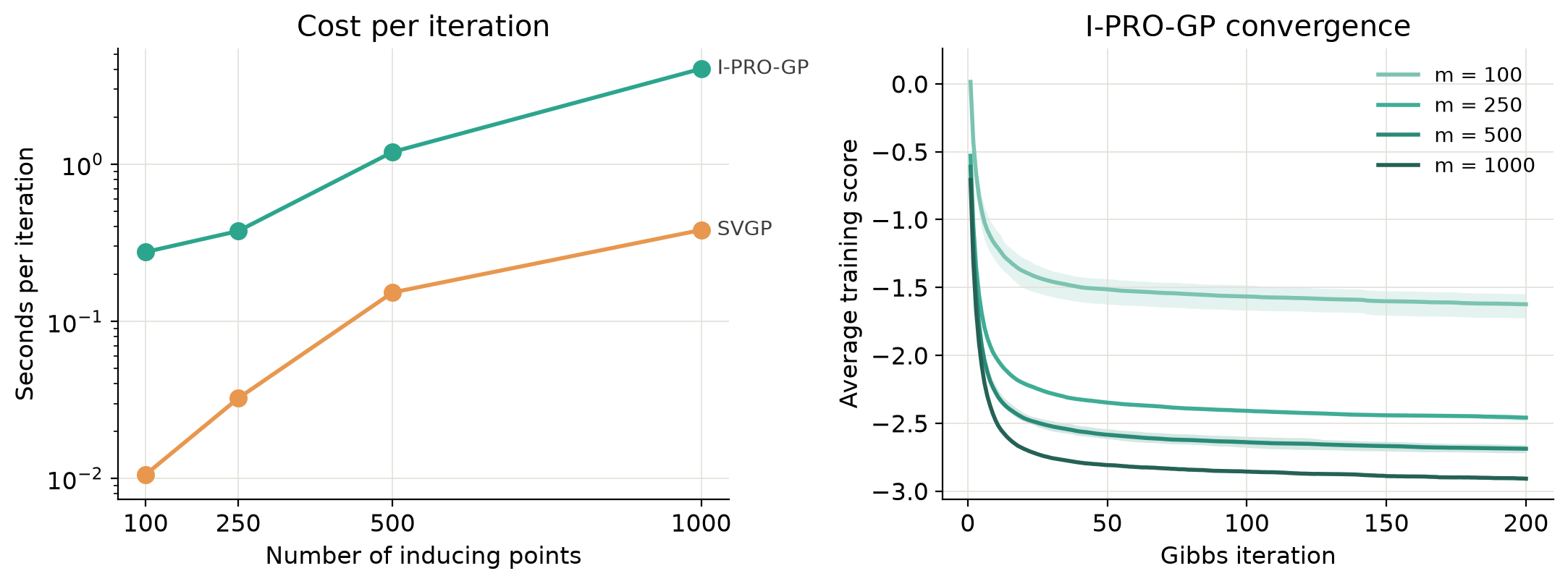}
    \caption{\IPROGP{} runtime analysis. Left: Cost per iteration for \IPROGP{} and \SVGP{} with batch size of 512. Right: \IPROGP{} average training score vs Replica--Gibbs iteration}
    \label{fig:runtime}
\end{figure}

\end{document}